\documentclass[11pt]{article}

\usepackage[final]{acl}

\usepackage{times}
\usepackage{latexsym}

\usepackage[T1]{fontenc}

\usepackage[utf8]{inputenc}

\usepackage{microtype}

\usepackage{inconsolata}

\usepackage{graphicx}
\usepackage{amsmath}
\usepackage{amsthm}
\usepackage{bm}
\usepackage{amssymb}
\usepackage{booktabs}
\usepackage{multicol}
\usepackage{multirow}
\usepackage[table]{xcolor}
\usepackage{makecell}
\usepackage{hyperref}
\usepackage{tcolorbox}
\usepackage{algorithm}
\usepackage{algpseudocode}

\definecolor{nmgray}{RGB}{245,245,245}

\newtheorem{theorem}{Theorem}[section]
\newtheorem{lemma}[theorem]{Lemma}

\newtheorem{corollary}[theorem]{Corollary}

\newtheorem{definition}{Definition}[section]
\newtheorem{assumption}{Assumption}[section]

\usepackage{enumitem}
\usepackage{xspace}
\newcommand{\mname}{{\fontfamily{lmtt}\selectfont \textbf{LCA}}\xspace}
\title{The Weakest Link Tells It All: Outcome-Supervised Process Reward Modeling via Learnable Credit Assignment}

\newcommand{\equalcontrib}{\textsuperscript{*}}
\newcommand{\corrauthor}{\textsuperscript{\textdagger}}

\author{
\textbf{
Tianyu Jia\textsuperscript{1,2,3}\equalcontrib,
Yue Fang\textsuperscript{1,2,3}\equalcontrib,
Hongxin Ding\textsuperscript{1,2,3},
Rihong Qiu\textsuperscript{1,2,3},
Zhibang Yang\textsuperscript{1,2,3},}\\
\textbf{
Zhijing Wu\textsuperscript{4}, 
Junfeng Zhao\textsuperscript{2,3}\corrauthor,
Yasha Wang\textsuperscript{1,3,5}\corrauthor,
Xu Chu\textsuperscript{2,3,6}\corrauthor
}
\\
\textsuperscript{1}National Engineering Research Center of Software Engineering, Peking University, Beijing, China\\
\textsuperscript{2}School of Computer Science, Peking University, Beijing, China\\
\textsuperscript{3}Key Laboratory of High Confidence Software Technologies, Ministry of Education, Beijing, China\\
\textsuperscript{4}GRG Banking Equipment Co., Ltd., Guangzhou, China\\
\textsuperscript{5}Peking University Information Technology Institute (Tianjin Binhai), Tianjin, China\\
\textsuperscript{6}Center on Frontiers of Computing Studies, Peking University, Beijing, China\\
{\normalsize
  \{jiatianyu,yuefang25\}@stu.pku.edu.cn,
  \{zhaojf,wangyasha,chu\_xu\}@pku.edu.cn
}
}

\begin{document}
\maketitle
\begingroup
\renewcommand{\thefootnote}{\fnsymbol{footnote}}
\footnotetext[1]{Equal contribution.}
\footnotetext[2]{Corresponding authors.}
\endgroup
\begin{abstract}
Process reward models (PRMs) enhance the reasoning capabilities of large language models (LLMs)
by providing fine-grained feedback, yet training PRMs typically requires expensive stepwise annotations. 
Outcome-supervised PRMs offer a scalable alternative by learning from final-answer correctness alone, but this introduces a fundamental \emph{credit assignment} challenge, i.e.,
attributing outcomes to responsible reasoning steps.
Existing approaches rely on either uniform or causal assignment,
both of which fail to anchor credit in step correctness and thus hinder process error identification.

In this work, we propose Outcome-Supervised Process Reward Modeling via \underline{L}earnable \underline{C}redit \underline{A}ssignment (\mname{}), an outcome-supervised PRM framework that jointly learns credit assignment and reward modeling under the principle of \emph{Weakest Link Assignment: a reasoning chain is as strong as its weakest link}.
To address mutual dependence between credit assignment and reward modeling, we formalize outcome-supervised PRM as a Multiple Instance Learning (MIL) problem and introduce Softmax-Weighted-Sum (SWS) pooling, an MIL pooling technique tailored for strong dependence and redundancy among reasoning states. 
We establish instance-level Bayes consistency of the resulting learning
objective under mild assumptions.
Extensive experiments demonstrate that 
\mname{} outperforms state-of-the-art outcome-supervised PRMs across multiple tasks and backbones. Code is available at \href{https://github.com/JiaTianyu20031016/MIL.git}{https://github.com/JiaTianyu20031016/MIL.git}.
\end{abstract}

\section{Introduction}

Process reward modeling (PRM) plays a crucial role in reasoning-intensive tasks for large language models (LLMs), such as 
mathematical problem-solving ~\cite{shao2024deepseekmathpushinglimitsmathematical, yu2024metamath} and code generation~\cite{li-etal-2025-codeprm, zhang2025dreamprmcodefunctionasstepprocessreward}.
Unlike outcome reward models (ORMs) ~\cite{cobbe2021trainingverifierssolvemath}, which evaluate only the final answer, 
PRMs~\cite{mathshepherd, letsverifystepbystep, luo2024improvemathematicalreasoninglanguage} assess the correctness of each intermediate reasoning step, providing dense, step-level signals that enhance test-time scaling~\cite{mathshepherd, snell2025scaling} and reinforcement learning~\cite{rewardingprogress, cheng2026stop}.

However, training PRMs requires step-level correctness labels, which are prohibitively expensive to collect from humans~\cite{letsverifystepbystep}. A natural alternative is to use final-answer correctness, cheap to verify and widely available, as ground-truth supervision. To recover step-level rewards from this trajectory-level signal, one must solve a \emph{credit assignment} problem: attributing trajectory-level correctness to the steps that cause it. This setting is termed \emph{outcome-supervised PRM}.

\begin{figure}[t]
  \includegraphics[width=\columnwidth]{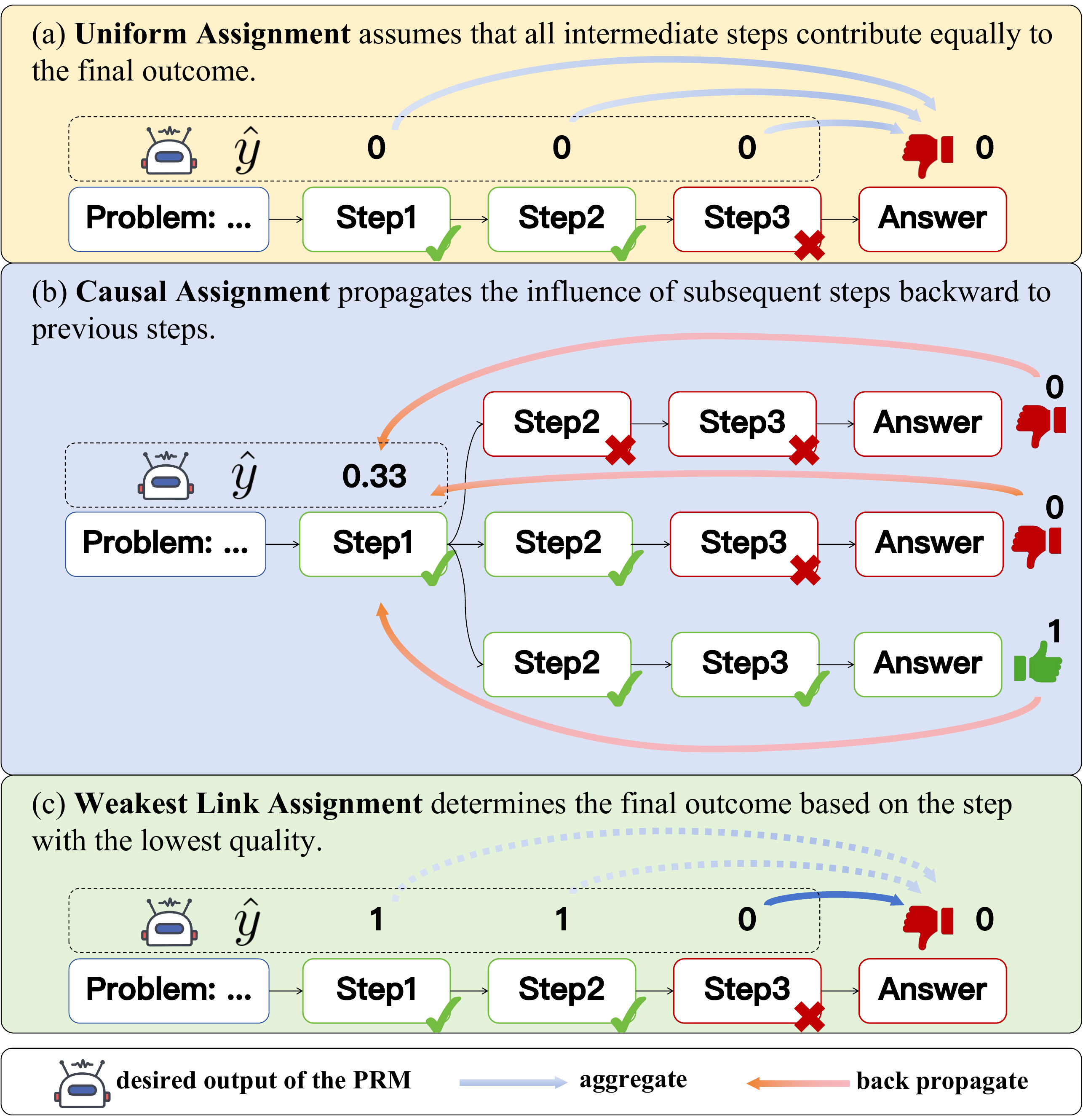}
  \caption{Illustration of different credit assignment strategies.
  }
  \label{fig:intro}
\end{figure}

Existing methods mainly fall into two perspectives.
\begin{itemize}
[leftmargin=*,itemsep=2pt,parsep=0pt,topsep=2pt,partopsep=0pt]
    \item \textit{\textbf{Uniform Assignment} via Additive Aggregation}: As illustrated in Fig.~\ref{fig:intro}(a), this line~\cite{chen2025discriminativepolicyoptimizationfortokenlevelrewardmodels,implicitprm} aggregates step-level scores by summation or averaging into a single trajectory score for optimization. Every step thus receives equal credit, regardless of which is actually responsible for the outcome.
    \item \textit{\textbf{Causal Assignment} via Rollout Estimation}: As illustrated in Fig.~\ref{fig:intro}(b), this line~\cite{mathshepherd,luo2024improvemathematicalreasoninglanguage,rstarmath} pseudo-labels each step via Monte Carlo estimation of the probability of reaching a correct final answer. Each step's score thus reflects the expected success of its downstream rollouts rather than its own correctness, making the actual error hard to localize. The estimation further depends on rollout policy and sampling design, which risks introducing bias and noise.
\end{itemize}

Both approaches share a common deficiency: neither\textbf{ attributes credit based on whether each step is actually correct}. Uniform Assignment treats every step the same; Causal Assignment lets a step's score depend on what comes after it. We propose to anchor credit assignment directly in step correctness, drawing on a long-standing intuition: a reasoning chain is only as strong as its weakest link~\cite{prasad2023receval, jacovi-etal-2024-AChain-of-ThoughtIsasStrongasItsWeakestLink, li-etal-2023-MakingLanguageModelsBetterReasonerswithStepAwareVerifier, gilda2026StructuredAbductive-Deductive-InductiveReasoningforLLMsviaAlgebraicInvariants}. A chain is correct if and only if every step is correct, and incorrect if and only if any step is wrong. We term this \emph{Weakest Link Assignment} (Figure~\ref{fig:intro}(c)). 
Training under this principle directly targets the same first-error criterion widely adopted in PRM evaluation and applications~\cite{cheng2026stop, letsverifystepbystep, zheng-etal-2025-processbench}, thereby closing the training–test gap and producing a PRM that accurately attributes credit to the outcome-determining steps.

Yet applying this principle to training poses a chicken-and-egg dilemma: assigning credit to specific steps requires knowing which step is wrong, while identifying wrong steps is precisely what the PRM is being trained to do. \textit{Credit assignment} and \textit{reward modeling} are thus mutually dependent.

To resolve this challenge, we propose Outcome-Supervised Process Reward Modeling via \underline{L}earnable \underline{C}redit \underline{A}ssignment (\mname{}), an outcome-supervised PRM framework that jointly learns credit assignment and reward modeling. 
Through a structural reduction, we formalize outcome-supervised PRM as a Multiple Instance Learning (MIL) problem~\cite{MILsurvey}: each trajectory prefix is an ``instance'', the whole trajectory is a ``bag'' (i.e., a set of instances), and the Weakest Link principle is the Standard Multi-Instance Assumption verbatim. This formalization places outcome-supervised PRM within the principled MIL framework, which jointly learns instance classification and the bag-to-instance attribution under bag-level supervision, thereby dissolving the chicken-and-egg dilemma between credit assignment and reward modeling. Through Softmax-Weighted-Sum (SWS) pooling, we resolve the degeneration of mainstream MIL methods under inter-instance dependence and redundancy, and embed a smoothed Weakest Link Assignment as a structural prior, yielding a learning objective that is Bayes consistent at the instance level.

To summarize, our contributions are threefold:
\begin{itemize}[leftmargin=*,itemsep=2pt,parsep=0pt,topsep=2pt,partopsep=0pt]
    \item \textbf{Insightfully}, through a systematic analysis of existing outcome-supervised PRM methods, we identify a shared deficiency: neither \textit{Uniform Assignment} nor \textit{Causal Assignment} anchors credit in step correctness. We propose \textit{Weakest Link Assignment} as a principled alternative, yielding PRMs that accurately attribute credit to the steps that cause the outcome.

    \item \textbf{Technically}, we propose \mname{}, an outcome-supervised PRM framework that (1) \textit{formalizes outcome-supervised credit assignment as a Multiple Instance Learning (MIL) problem} under the Weakest Link principle; (2) \textit{introduces Softmax-Weighted-Sum (SWS) pooling}, an MIL pooling technique that adaptively reweights reasoning steps and embeds a smoothed Weakest Link Assignment as a structural prior, addressing the degeneration of mainstream MIL methods under inter-instance dependence and redundancy; and (3) \textit{proves Bayes consistency} of the resulting learning objective at the instance level.

    \item \textbf{Experimentally}, on error identification and best-of-N test-time scaling, \mname{} outperforms current state-of-the-art outcome-supervised PRM baselines across multiple LLM backbones and generator models.
    Ablation studies further validate SWS pooling against widely used MIL pooling techniques and characterize the role of the temperature parameter.
\end{itemize}

\section{Preliminaries}
\subsection{Multiple Instance Learning}\label{sec:mil}

Multiple Instance Learning (MIL) is a weakly-supervised learning method designed for settings where only coarse-grained labels are available but fine-grained classification is required. MIL assumes a hierarchical structure of instances and bags. Let $\mathcal{X}$ be the instance space. A bag $B$ is a finite set of instances $x_i \in \mathcal{X}$. The bag space is $\mathcal{B}$, with a data distribution $\mathbb{P}$ over $\mathcal{B}$.

In the standard binary MIL setting, each instance and each bag has a label $y \in \{0,1\}$ (0 = negative, 1 = positive). Instance labels are unobserved; only bag labels are given. Given a training set $D = \{ \langle B^k, y(B^k) \rangle \}_{k=1}^K$ drawn from $\mathbb{P}$, the goal is to learn an instance-level classifier $f: \mathcal{X} \to [0,1]$.

Most MIL algorithms rely on the \emph{Standard Multi-Instance Assumption}~\cite{standardMILassumption}:
\begin{assumption}[Standard Multi-Instance Assumption]
    A bag is positive if and only if it contains at least one positive instance:
    \begin{align*}
        y(B) = \max\{y(x_1),\dots,y(x_n)\}, \\
        \quad \forall B = \{x_1,\dots,x_n\} \in \mathcal{B}.
    \end{align*}
\end{assumption}
This assumption holds in many real-world applications, such as anomaly detection in long time series or whole-slide images (WSIs)~\cite{milforimage,milfortimeseries}. Anomalies typically appear in local regions, yet the label is assigned to the entire input. The goal of MIL in such cases is to pinpoint the anomaly location.

Modern MIL methods are predominantly pooling-based. These methods develop instance-level classifiers and use a pooling function to aggregate instance predictions into a bag-level prediction. Standard supervised learning optimizes the bag-level prediction, typically using a cross-entropy loss:
\begin{align*}
 L_{CE}\Big(y(B), \text{Pool}(\{p_1,\dots,p_n\})\Big), 
\end{align*}
where $p_1,\dots,p_n$ are instance predictions and $\text{Pool}$ is the pooling function.
In this way, instance classification and bag-to-instance attribution are jointly optimized in an end-to-end manner.

A more detailed discussion of MIL pooling techniques is deferred to Appendix~\ref{app:pooling}.

\subsection{Process Reward Model}

\paragraph{Reasoning trajectory.}
A reasoning trajectory $\tau$ consists of a question $q$ followed by a sequence of steps $s_1, \dots, s_T$:
\[
\tau := \langle q, s_1, \dots, s_T \rangle,
\]
where each step advances the reasoning based on the question and all previous steps.\footnote{For clarity, we consider only sequential reasoning without self-correction in the main text. 
Our method naturally applies to reasoning with reflection and backtracking 
through a simple conceptual transformation; see Appendix~\ref{app:reflection}.}
An incomplete trajectory is termed a \emph{prefix}:
\[
\tau_{:t} := \langle q, s_1, \dots, s_t \rangle, \quad t \le T.
\]

\paragraph{Correctness of reasoning prefixes.}
A reasoning step is logically erroneous if its conclusion is false or does not follow from the preceding context. Once such an error occurs, subsequent steps inherit the error, making their individual correctness ill-defined~\cite{letsverifystepbystep}. We therefore assign correctness labels to prefixes rather than to isolated steps.

Formally, define correctness of a prefix $\tau_{:t}$ as
\[
y(\tau_{:t}) \in \{0,1\},
\]
where $1$ denotes that $\tau_{:t}$ contains logical error(s) (positive) and $0$ denotes no error (negative), matching the standard MIL convention. By this definition, error monotonicity follows immediately:
\begin{equation}\label{eq:label-monotonic}
    y(\tau_{:t+1}) \ge y(\tau_{:t}),
\end{equation}
because any extension of an erroneous prefix remains erroneous. The correctness of the full trajectory is thus
\begin{equation}\label{eq:weakest-link}
    y(\tau) = y(\tau_{:T}) = \max\{y(\tau_{:1}), \dots, y(\tau_{:T})\},
\end{equation}
which is precisely the Weakest Link Assignment.

\paragraph{Process reward modeling.}
Process reward modeling learns a classifier $f$ that maps each prefix to an error probability: $f(\tau_{:t}) \in [0,1]$,
with the goal that $f(\tau_{:t}) \approx y(\tau_{:t})$. This is equivalent to \emph{first-error detection}~\cite{letsverifystepbystep, zheng-etal-2025-processbench}: the first prefix for which $f$ signals an error identifies the earliest step that caused the failure.
\section{Analysis}

\begin{figure*}[t]
  \includegraphics[width=\textwidth]{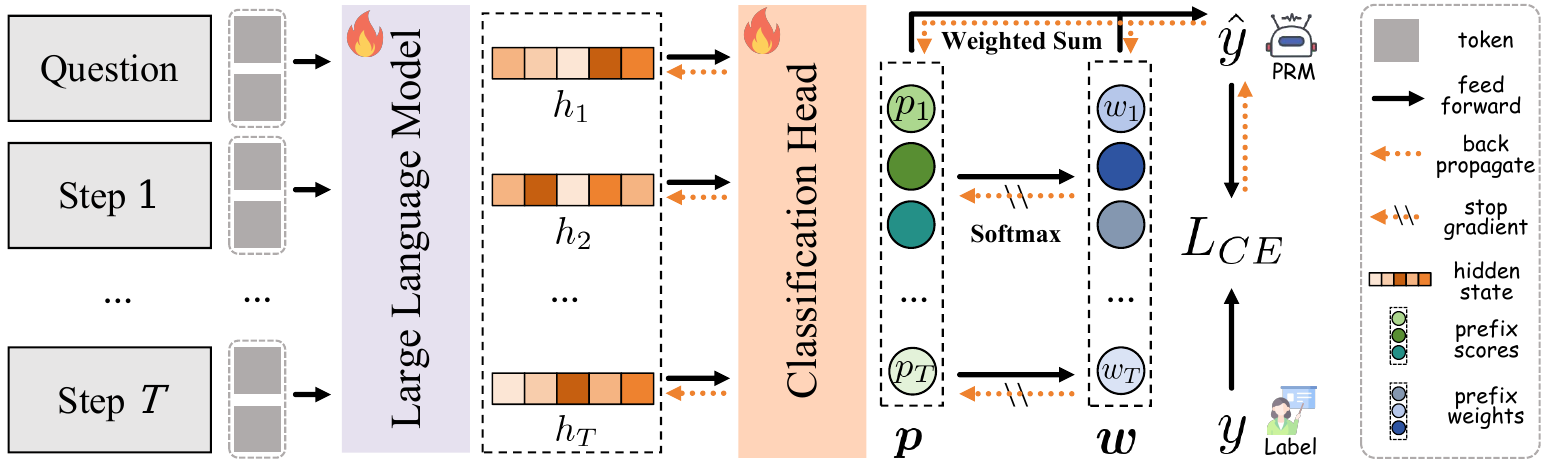}
  \caption{Overview of \mname{}. The PRM outputs per-prefix error probabilities $\bm{p}$, which are then aggregated via Softmax-Weighted-Sum (SWS) pooling. The PRM is trained end‑to‑end with a trajectory‑level cross‑entropy loss.}
  \label{fig:method}
\end{figure*}

\subsection{Process Reward Modeling as Multiple Instance Learning}

Under the Weakest Link principle, we formalize outcome-supervised process reward modeling as a Multiple Instance Learning (MIL) problem:
\begin{itemize}
[leftmargin=*,itemsep=2pt,parsep=0pt,topsep=2pt,partopsep=0pt]
    \item \textbf{Data structure.} A full trajectory $\tau$ corresponds to a bag, and each prefix $\tau_{:t}$ to an instance.
    We thus use ``bag'' and ``trajectory'' (resp.\ ``instance'' and ``prefix'') interchangeably when the context is clear. 
    \item \textbf{Assumption.} By the definition of prefix correctness, the bag-level label is the maximum over instance labels (Eq.~\ref{eq:weakest-link});
    this exactly matches the Standard Multi-Instance Assumption. 
    \item \textbf{Learning objective.} Given a dataset with trajectory-level labels: $D = \{\langle\tau^k, y(\tau^k)\rangle\}_{k=1}^K$,
    our goal is to learn a prefix-level classifier $f$ such that $f(\tau_{:t}) \approx y(\tau_{:t})$, consistent with the paradigm of MIL.
\end{itemize}

Despite this formal alignment, outcome-supervised PRM raises a unique challenge absent in classic MIL: \textbf{prefixes within a trajectory are both highly dependent and redundant, violating the common assumption of instance independence or weak dependence}~\cite{MILsurvey}. This violation stems from three structural properties of reasoning trajectories:
\begin{itemize}
[leftmargin=*,itemsep=2pt,parsep=0pt,topsep=2pt,partopsep=0pt]
    \item \textbf{Autoregressive}: each reasoning step is conditioned on all previous steps.
    \item \textbf{Nested}: each prefix includes previous ones.
    \item \textbf{Monotonic}: prefix labels are monotonic (Eq.~\ref{eq:label-monotonic}).
\end{itemize}
Together, these three properties render each instance highly informative about the others. Most critically, the last prefix alone suffices to determine the bag label, making all earlier prefixes ignorable for bag-level classification.

Consequently, conventional MIL pooling techniques struggle in our setting. 
In preliminary experiments, we applied max-pooling-based MIL, 
the most direct realization of the Standard Multi-Instance Assumption. 
Yet the PRM rapidly degenerated into an ORM,
producing meaningful predictions only at the final prefix while outputting a trivial value for all preceding prefixes. Detailed setup and results are provided in Appendix~\ref{app:preliminary-experiments}. 
Attention-based MIL methods, though most prevalent in the deep learning era, are prone to a similar failure mode.
The model could focus its attention exclusively on the last prefix, thus offering no guarantee of instance-level classification performance.

\subsection{Multiple Instance Learning with Softmax-Weighted-Sum Pooling}

To handle the strong dependence and redundancy among instances, we propose Softmax-Weighted-Sum (SWS) pooling. Given instance-level predictions 
$\bm{p} = (p_1,\dots,p_T) \in [0,1]^T$, the bag-level prediction is
\begin{equation}
    \text{SWS}_\alpha(\bm{p}) = \sum_{t=1}^T w_t p_t,\quad
    w_t = \frac{e^{p_t/\alpha}}{\sum_{j=1}^T e^{p_j/\alpha}},
\end{equation}
where the temperature $\alpha>0$ controls the entropy of the weight distribution. 
As $\alpha\to 0$, SWS approaches max pooling; as $\alpha\to\infty$, it approaches 
average pooling.

SWS pooling offers two key properties for our setting.
First, because every weight $w_t$ admits a strictly positive lower bound under bounded bag sizes, each instance is guaranteed to receive a non-negligible credit. This prevents the model from basing its predictions on an overly narrow subset of instances and thus promotes more generalizable instance classification.
Second, the softmax weighting naturally encodes the Standard Multi-Instance Assumption as a structural prior: the weight of each instance grows with its predicted positivity, aligning with the Weakest Link Assignment where positive instances dominate the bag prediction.

SWS-pooling-based MIL provides an adaptive yet principled credit assignment mechanism. The model learns instance predictions $p_t$, which in turn determine the softmax weights $w_t$ and thus each instance’s contribution to the bag-level 
loss. In this way, the model \emph{learns} which instances drive the optimization objective under pure bag-level supervision, while the softmax-based architecture embeds the Standard MIL assumption to constrain the solution space.

\subsection{Theoretical Analysis}\label{sec:theory}
We provide theoretical results about the Bayes consistency of SWS pooling. The analysis applies not only to our settings of outcome-supervised PRM but to any MIL setting that satisfies Assumption~\ref{ass:M}. We therefore adopt the general MIL notation introduced in Section~\ref{sec:mil}.
\begin{definition}[Bag cover set]\label{def:bag_cover_set}
    Let $S \subseteq \mathcal{X}$ be a set of instances. 
    The \emph{bag cover set} of $S$ is the set of all bags that contain at least one instance from $S$:
    \[
        \mathcal{C}(S) := \{ B \in \mathcal{B} \mid B \cap S \neq \emptyset \}.
    \]
    Similarly, the \emph{positive bag cover set} $\mathcal{C}^+(S)$ is the set of all positive bags containing at least one instance from $S$; the \emph{negative bag cover set} $\mathcal{C}^-(S)$ is defined analogously.
\end{definition}
\begin{assumption}[$M$-Assumption]\label{ass:M}
    There exists a constant $M > 0$ such that for any set $S$ of negative instances, if
    \[
        \mathbb{P}(\mathcal{C}(S)) = \mathbb{P}(\mathcal{C}^+(S)) + \mathbb{P}(\mathcal{C}^-(S)) > 0,
    \]
    then
    \[
        \frac{\mathbb{P}(\mathcal{C}^+(S))}{\mathbb{P}(\mathcal{C}^-(S))} \le M.
    \]
\end{assumption}
Assumption~\ref{ass:M} is a mild assumption that prevents the probability ratio of a negative instance appearing in positive bags versus negative bags from being arbitrarily large.
\begin{theorem}[Unique global minimizer]\label{thm:mil-minimize}
Let Assumption~\ref{ass:M} hold. 
Let $N$ be the maximum allowed bag size.
Let $f(x) \in [0,1]$ be an instance classifier. 
For a bag $B = (x_1,\dots,x_n)$ with label $y(B) \in \{0,1\}$, define the softmax weighted sum
\[
\hat{p}(B)=\text{SWS}_\alpha(f(x_1), \cdots, f(x_n)),
\]
and the expected loss
\[
\mathcal{L}(f)=\mathbb{E}_{B\sim\mathbb{P}} L_{\text{CE}}(\hat{p}(B), y(B)).
\]
Then there exists a sufficiently small temperature $\alpha > 0$ such that, up to a set of instances $S$ whose bag cover set has probability zero, $\mathcal{L}$ is uniquely minimized by the true label function:
\[
f^*(x) = y(x) \in \{0,1\}, \qquad \forall x \notin S.
\]
\end{theorem}
\begin{theorem}[Convergence in probability]\label{thm:mil-converge}
Let Assumption~\ref{ass:M} hold. 
Let $N$ be the maximum allowed bag size.
Let $f^*(x) = y(x)$ be the Bayes optimal classifier, and let $\{f_l\}$ be a sequence of instance classifiers.
Define the expected loss $\mathcal{L}$ as in Theorem~\ref{thm:mil-minimize}.
Then there exists a sufficiently small temperature $\alpha > 0$ such that, if $\{\mathcal{L}(f_l)\}$ converges to $\mathcal{L}(f^*)$:
\[
\lim_{l \to \infty} \mathcal{L}(f_l)
=
\mathcal{L}(f^*)
=
\inf_f \mathcal{L}(f),
\]
then $\{f_l\}$ converges to $f^*$ in probability.
\end{theorem}
Together, Theorem~\ref{thm:mil-minimize} and Theorem~\ref{thm:mil-converge} guarantee a desirable property of SWS-pooling-based MIL algorithms: by minimizing the expected loss $\mathcal{L}$, the instance classifier $f$ approximates the Bayes optimal classifier $f^*$. Consequently, even without instance-level supervision, the learned classifier is expected to perform well on instance classification.

Our theoretical analysis reveals the central role of temperature $\alpha$ in SWS pooling. Intuitively, a larger $\alpha$ makes the credit assignment more uniform, risking misclassification of negative instances that frequently appear in positive bags. So Bayes consistency requires $\alpha$ to be sufficiently small, as stated in Theorem~\ref{thm:mil-minimize} and Theorem~\ref{thm:mil-converge}.
On the other hand, however, an excessively small $\alpha$ concentrates the credit assignment on a single instance, thereby starving co-occurring positive instances of adequate optimization signals. Consequently, the choice of $\alpha$ involves a trade-off between theoretical soundness and empirical effectiveness.

A more formal statement of the definitions, assumptions, theorems, and their detailed proofs can be found in Appendix~\ref{app:theorem}.

\section{Method}
\label{sec:method}
Building on the above analysis, we propose
\mname{}, an outcome-supervised PRM framework that jointly learns credit assignment and reward modeling using an SWS-pooling-based MIL algorithm.
Figure~\ref{fig:method} illustrates the overall pipeline.

We adopt the standard discriminative PRM architecture~\cite{zheng2026surveyprocessrewardmodels},
which appends a binary classification head to a pretrained LLM backbone.
Given a reasoning trajectory
$\tau=\langle q,s_1,\dots,s_T\rangle$, we serialize the prompt and
reasoning steps as
\begin{equation}
x(\tau)
=
q \oplus d
\oplus s_1 \oplus d
\oplus \cdots
\oplus s_T \oplus d,
\end{equation}
where $\oplus$ denotes textual concatenation in the indicated order and
$d$ denotes the delimiter between the prompt and successive reasoning
steps. Unless otherwise specified, following recent PRM data-processing
conventions~\cite{zhang-etal-2025-lessons,zheng-etal-2025-processbench},
we instantiate $d$ as the double-newline sequence
\texttt{\textbackslash n\textbackslash n}.
The LLM backbone $\mathcal{M}$ processes the serialized trajectory in a
single causal forward pass. Let $e_t$ denote the token position of the
last token of the delimiter immediately following step $s_t$. We use the
final-layer hidden state at this position as the representation of prefix
$\tau_{:t}$:
\begin{equation}
    \bm{h}_t = [\mathcal{M}(x(\tau))]_{e_t},
    \qquad
    p_t = \psi(\bm{h}_t),
\end{equation}
where $\psi$ is the binary classification head and $p_t\in[0,1]$ is the
predicted probability that prefix $\tau_{:t}$ is incorrect. Thus,
\begin{equation}
    f(\tau)=(p_1,\dots,p_T)\in[0,1]^T.
\end{equation}

We train the PRM end-to-end by minimizing the empirical cross-entropy loss at trajectory level:
\begin{equation}
\begin{aligned}
\hat{\mathcal{L}}_K
&= \frac{1}{K}\sum_{k=1}^{K}
L_{CE}\bigl(\hat{p}(\tau^k), y(\tau^k)\bigr) \\
&= -\frac{1}{K} \sum_{k=1}^{K}
\Bigl(
y(\tau^k)\log \hat{p}(\tau^k) \\
&\qquad\qquad
+ (1-y(\tau^k))\log(1-\hat{p}(\tau^k))
\Bigr).
\end{aligned}
\end{equation}
where $y(\tau^k)\in\{0,1\}$ indicates whether the final answer is correct ($0$)
or incorrect ($1$). No step-level supervision is required.
In practice, we apply a stop-gradient operation to the softmax weights
$w_t$ during backward propagation to stabilize training; for a justification
of this operation and a controlled ablation, see
Appendix~\ref{app:stop-gradient}.
\begin{table*}[t]
\centering
\fontsize{7.5pt}{8pt}\selectfont
\resizebox{\textwidth}{!}{
\begin{tabular}{lccccc}
\toprule
\textbf{Model}
& \textbf{GSM8K} 
& \textbf{MATH} 
& \textbf{OlympiadBench} 
& \textbf{Omni-MATH} 
& \textbf{Avg.} \\
\midrule
\multicolumn{6}{c}{\textit{Large Language Models as Critic}} \\
\midrule
Llama-3-8B-Instruct           & 13.1 & 13.8 & 4.8  & 12.6 & 11.1 \\
Llama-3.1-8B-Instruct         & 10.9 & 5.1  & 2.8  & 1.6  & 5.1  \\
Qwen2.5-7B-Instruct           & 36.5 & 36.6 & 29.7 & 27.4 & 32.6 \\
Qwen2.5-Math-7B-Instruct      & 26.8 & 25.7 & 14.2 & 12.7 & 19.9 \\
Qwen2.5-Coder-7B-Instruct     & 14.3 & 6.5  & 4.1  & 1.8  & 6.7  \\
\midrule
\multicolumn{6}{c}{\textit{Open Source Process Reward Models}} \\
\midrule
Math-Shepherd-PRM-7B          & 47.9 & 29.5 & 24.8 & 23.8 & 31.5 \\
RLHFlow-PRM-Mistral-8B        & 50.4 & 33.4 & 13.8 & 15.8 & 28.4 \\
RLHFlow-PRM-Deepseek-8B       & 38.8 & 33.8 & 16.9 & 16.9 & 26.6 \\
EurusPRM-Stage1               & 44.3 & 35.6 & 21.7 & 23.1 & 31.2 \\
EurusPRM-Stage2               & 47.3 & 35.7 & 21.2 & 20.9 & 31.3 \\
Qwen2.5-Math-7B-Math-Shepherd-PRM & 62.5 & 31.6 & 13.7 & 7.7  & 28.9 \\
\midrule
\multicolumn{6}{c}{\textit{Process Reward Models Trained on \textbf{Llama-3.2-3B-Instruct}}} \\
\midrule
ImplicitPRM & \textbf{54.8} & 34.1 & \underline{22.2} & 12.4  & \underline{30.9} \\
MathShepherd & 38.8 & 22.7 & 14.5 & 11.4  & 21.8 \\
SCAN & 47.6 & 27.6 & 16.2 & 12.2  & 25.9 \\
OmegaPRM & 38.7 & 30.1 & 18.2 & 14.1  & 25.3 \\
PQM & 34.4 & \underline{38.3} & 20.2 & \underline{27.9}  & 30.2 \\
\rowcolor{blue!10}
\mname{} (Ours, $\alpha=1$) & \underline{48.4} & \textbf{42.1} & \textbf{30.3} & \textbf{29.4} & \textbf{37.6} \\
\midrule
\multicolumn{6}{c}{\textit{Process Reward Models Trained on \textbf{Qwen2.5-Math-7B-Instruct}}} \\
\midrule
ImplicitPRM & 45.2 & 36.3 & 23.5 & \textbf{25.1}  & 32.5 \\
MathShepherd & 59.2 & 31.1 & 12.7 & 8.1  & 27.8 \\
SCAN & \underline{60.6} & 33.9 & 15.0 & 6.8  & 29.1 \\
OmegaPRM & 44.8 & 37.1 & 22.9 & 18.4  & 30.8 \\
PQM & 55.7 & \underline{43.9} & \textbf{28.5} & \underline{24.5}  & \underline{38.1} \\
\rowcolor{blue!10}
\mname{} (Ours, $\alpha=1$) & \textbf{65.7} & \textbf{48.7} & \underline{26.2} & 19.9  & \textbf{40.2} \\
\bottomrule
\end{tabular}
}
\caption{F1 scores of LLM-as-Critic models and PRMs on ProcessBench.}
\label{tab:processbench}
\end{table*}

\begin{table*}[t]
\centering
\fontsize{7.5pt}{8pt}\selectfont
\resizebox{\textwidth}{!}{
\begin{tabular}{@{}l|ccccccccc|c@{}}
\toprule

\multirow{2}{*}{\bf Reward Model}
& \multicolumn{3}{c|}{\makecell{\bf Mistral-7B-Inst-v0.2\\ \bf Pass@1: 9.6}}
& \multicolumn{3}{c|}{\makecell{\bf Llama-3.1-8B-Inst\\ \bf Pass@1: 44.6}}
& \multicolumn{3}{c|}{\makecell{\bf Qwen-2.5-7B-Inst\\ \bf Pass@1: 65.8}}
& \multirow{2}{*}{\bf Avg.}
\\

\cmidrule(lr){2-10}

& @8 & @32 & @64
& @8 & @32 & @64
& @8 & @32 & @64
&
\\

\midrule

Majority Vote & 18.2 & 21.0 & 22.8 & 56.6 & 62.2 & 63.2 & 67.8 & 70.2 & 70.6 & 50.2 \\ 

\midrule

ImplicitPRM & 19.0 & 18.4 & 17.6 & \underline{59.4} & 61.0 & 62.8 & \textbf{70.4} & 68.2 & 69.2 & 49.6 \\ 

MathShepherd & 24.6 & 30.0 & 33.2 & 59.0 & \underline{65.6} & 65.4 & 67.2 & 69.6 & 69.6 & 53.8 \\ 

SCAN & \textbf{26.4} & \underline{31.0} & \underline{33.6} & \textbf{59.8} & 65.2 & \underline{66.4} & 67.4 & 69.2 & 69.2 & \underline{54.2} \\ 

OmegaPRM & 13.4 & 15.2 & 17.2 & 53.4 & 55.4 & 57.4 & 66.2 & 68.0 & 68.4 & 46.1 \\ 

PQM & 21.8 & 27.0 & 28.8 & 58.2 & 60.6 & 60.4 & \underline{69.8} & \underline{70.6} & \underline{70.4} & 52.0 \\ 

\rowcolor{blue!10}
\mname{} (Ours, $\alpha=1$) & \underline{25.0} & \textbf{31.2} & \textbf{35.4} & \underline{59.4} & \textbf{66.0} & \textbf{67.8} & 69.0 & \textbf{71.4} & \textbf{71.4} & \textbf{55.2} \\

\bottomrule
\end{tabular}
}
\caption{Different PRMs' best-of-N sampling performance on MATH-500 with three different generation models.
}
\label{tab:best_of_n}
\end{table*}

\section{Experiments}\label{sec:experiment}

We conduct a series of experiments to answer four research questions:
\begin{itemize}
[leftmargin=*,itemsep=2pt,parsep=0pt,topsep=2pt,partopsep=0pt]
    \item \textbf{RQ1}: Can PRMs trained with \mname{} reliably locate errors in reasoning trajectories?
    \item \textbf{RQ2}: Can PRMs trained with \mname{} enhance LLM reasoning capability via test-time scaling?
    \item \textbf{RQ3}: How does the temperature $\alpha$ affect SWS pooling's performance?
    \item \textbf{RQ4}: Does SWS pooling outperform other pooling techniques on process reward modeling?
\end{itemize}

Our experiments focus on math reasoning, where the correctness of intermediate steps and final answers is relatively well-defined and easy to verify. Unless otherwise stated, we use the correctness of the final answer as the trajectory-level label, without any process-level annotation.

Implementation details for all experiments are given in Appendix~\ref{app:implement-details}.

\subsection{Experiments on Error Identification (RQ1)}\label{sec:rq1}

\noindent\textbf{Datasets and Models.} We use the publicly available Math-Shepherd dataset~\cite{mathshepherd} as the training set. Though this dataset provides step-level pseudo labels generated by Monte Carlo estimation, we discard these process labels
and retain only the reasoning trajectories and their final-answer
correctness as trajectory-level supervision for \mname{}. We use Qwen2.5-Math-7B-Instruct~\cite{yang2024qwen25mathtechnicalreportmathematical} and Llama-3.2-3B-Instruct~\cite{meta2024llama32modelcard} as the LLM backbones for training PRMs.

\noindent\textbf{Evaluation Metrics.} We use ProcessBench~\cite{zheng-etal-2025-processbench} to measure the PRM’s ability to identify the first error location in a response. The benchmark evaluates recognition of fully correct responses and precise error identification in incorrect responses. We calculate accuracies on correct and erroneous samples and report their harmonic mean as the final F1 score.

\noindent\textbf{Selection of $\alpha$.}
Our training data provide only trajectory-level outcome labels and do
not contain ground truth step-level annotations. Consequently, a held-out split of
the training data would not support selecting $\alpha$ according to
step-level error-identification performance. Following ProcessBench's protocol of selecting a shared
hyperparameter on a designated subset%
~\cite{zheng-etal-2025-processbench}, we similarly select a single
shared value of $\alpha$ using one ProcessBench subset. Specifically,
using Qwen2.5-Math-7B-Instruct as the backbone, we search over
$\alpha\in\{0.1,0.2,\ldots,2.0\}$ and select the value achieving the
highest F1 score on the MATH split. This procedure yields
$\alpha=1.0$, which we subsequently fix without further tuning in all
main experiments (Sections~\ref{sec:rq1} and~\ref{sec:rq2}). Only the
MATH split is used for this selection; the other three ProcessBench
splits are not used for hyperparameter tuning.

\noindent\textbf{Baselines.} We compare our method with the current state-of-the-art baselines:
\begin{itemize}
[leftmargin=*,itemsep=2pt,parsep=0pt,topsep=2pt,partopsep=0pt]
    \item \textbf{ImplicitPRM}~\cite{implicitprm} parameterizes the reward model as a log-likelihood ratio of policy and reference models, training with trajectory-level correctness labels.
    \item \textbf{MathShepherd}~\cite{mathshepherd} uses brute-force Monte Carlo estimation to generate hard pseudo labels for every step.
    \item \textbf{SCAN}~\cite{ding2026scan} denoises Monte Carlo pseudo labels using the confidence of the data generator.
    \item \textbf{OmegaPRM}~\cite{luo2024improvemathematicalreasoninglanguage} uses Monte Carlo tree search to selectively generate high-quality reasoning trajectories with soft pseudo labels.
    \item \textbf{PQM}~\cite{pqm} formulates process reward modeling as a Q-value ranking problem and trains PRMs with a ranking loss, which requires estimation of the first error location.
\end{itemize}
We also incorporate several 7B-scale open-source PRMs and LLM-as-critic models as baselines; their results are taken directly from previous research~\cite{zhang-etal-2025-lessons}.

\paragraph{Main Results.}
Table~\ref{tab:processbench} reports evaluation results on ProcessBench. 
\textbf{\mname{} achieves the highest average F1 among the compared
outcome-supervised methods, with relative improvements of 21.7\% and
5.5\% over the second-best results for the 3B and 7B backbones, respectively.}
Notably, \mname{} avoids the Monte Carlo sampling required by
MathShepherd, SCAN, OmegaPRM, and PQM while achieving higher average F1
than these baselines.

More detailed results can be found in Table~\ref{tab:all_processbench} in Appendix~\ref{app:full-processbench-results}.

\subsection{Experiments on Test-Time Scaling (RQ2)}
\label{sec:rq2}
\noindent\textbf{Datasets and Models.}
We train PRMs on the Math-Shepherd dataset, using Qwen2.5-Math-7B-Instruct as the backbone. 
We evaluate test-time scaling on MATH-500~\cite{letsverifystepbystep}, a subset of MATH dataset. 
To test generalizability, we select three generator models of different sizes, families, and reasoning capabilities: Mistral-7B-Instruct-v0.2~\cite{jiang2023mistral7b}, Llama-3.1-8B-Instruct~\cite{grattafiori2024llama3herdmodels}, and Qwen2.5-7B-Instruct~\cite{qwen2025qwen25technicalreport}.

\noindent\textbf{Evaluation Metrics.}
Following prior work~\cite{mathshepherd,letsverifystepbystep,luo2024improvemathematicalreasoninglanguage}, we evaluate PRMs by their verification ability using best-of-N sampling. 
The metric best-of-N accuracy@n measures the correctness of the most preferred trajectory selected by the PRM from $n$ candidates per question. 

\noindent\textbf{Baselines.}
We compare our method with the same baselines as in Section~\ref{sec:rq1}:
ImplicitPRM, MathShepherd, SCAN, OmegaPRM, and PQM.

\noindent\textbf{Main Results.}
Table~\ref{tab:best_of_n} reports the best-of-N performance for $N=8$, $32$, and $64$.
\textbf{Our method achieves the highest average score across generators.} Moreover, \mname{} attains the best results on all three generation models at $N=32$ and $N=64$, demonstrating strong scalability and generalization.
Although our method is primarily designed for intra-trajectory credit assignment rather than for validating the final answer, it still delivers consistent gains on best-of-N accuracy.
We hypothesize that Weakest Link Assignment contributes to this behavior
by anchoring credit assignment in step correctness. This mechanism
may encourage the PRM to focus on local reasoning errors rather than relying on
spurious correlations or shortcuts, thereby mitigating
overfitting and improving generalization.

\subsection{Ablation Study (RQ3 \& RQ4)}
\label{sec:rq34}
We conduct ablation studies to examine the effect of temperature $\alpha$ and the advantage of SWS pooling over other common MIL pooling techniques. 

\paragraph{Effect of Temperature (RQ3).}
Following the settings of our main experiments, we train PRMs with different $\alpha$ using Qwen2.5-Math-7B-Instruct as the backbone and Math-Shepherd dataset as the training set, and report their F1 scores on ProcessBench splits.

\begin{figure}[t]
  \includegraphics[width=\columnwidth]{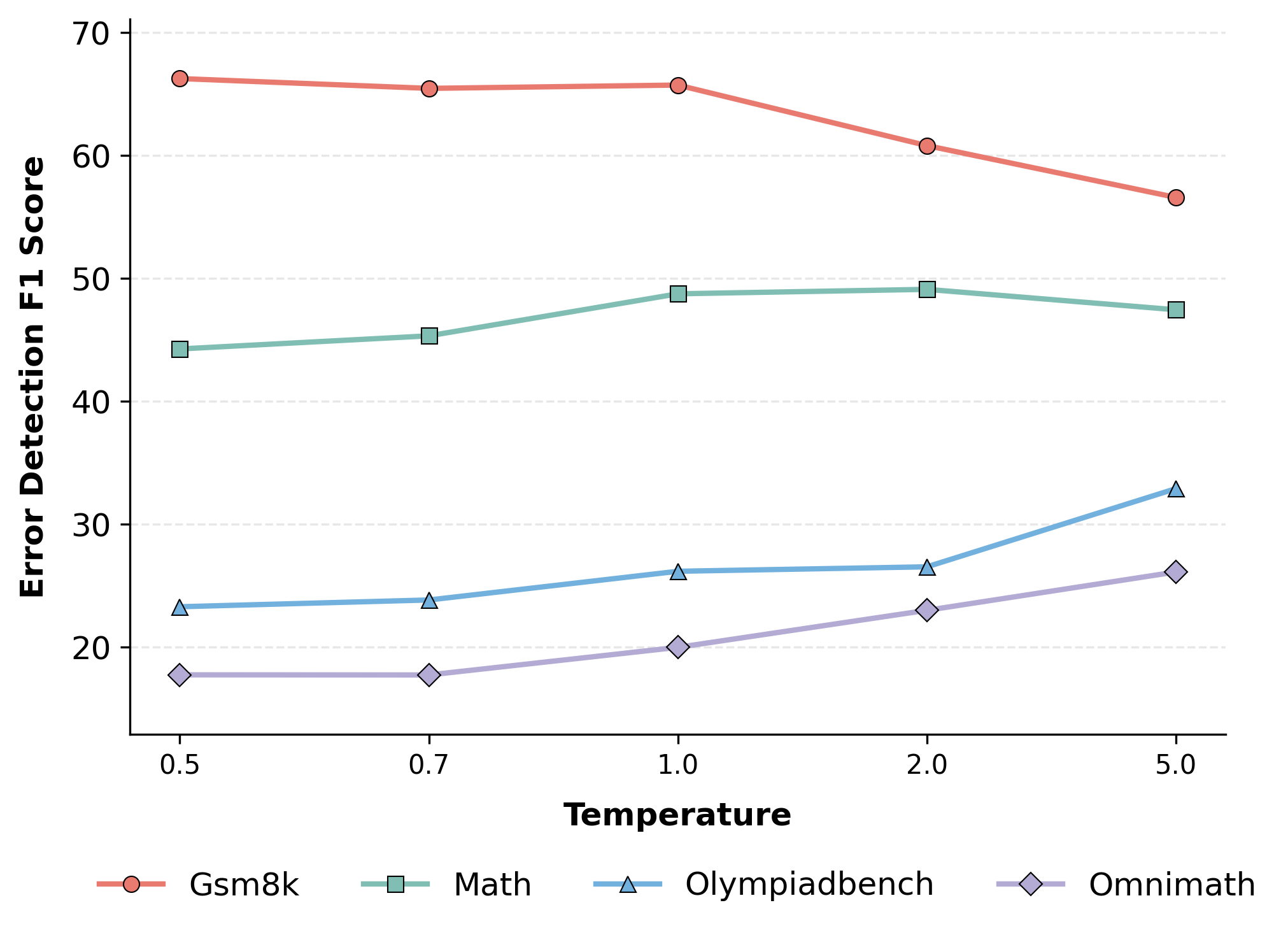}
  \caption{Performance on ProcessBench splits with varying temperature $\alpha$. }
  \label{fig:f1_vs_temperature}
\end{figure}

Figure~\ref{fig:f1_vs_temperature} shows the results. As $\alpha$ increases, performance on GSM8K decreases, while performance on OlympiadBench and Omni-Math increases. This finding aligns with our theoretical insights. A larger $\alpha$ makes credit assignment more uniform across instances, which encourages the model to identify as many positive instances as possible within a positive bag, and thus benefits process error detection. However, this also risks misclassifying negative instances that appear in positive bags, e.g., correct prefixes of eventually wrong trajectories. Consequently, as $\alpha$ increases, performance drops on simple datasets with fewer process errors (e.g., GSM8K) but improves on complex datasets with substantial process errors (e.g., OlympiadBench and Omni-Math)~\cite{zhang-etal-2025-lessons}. This again highlights that choosing $\alpha$ requires a trade-off between theoretical consistency and practical effectiveness.

\begin{figure}[t]
  \includegraphics[width=\columnwidth]{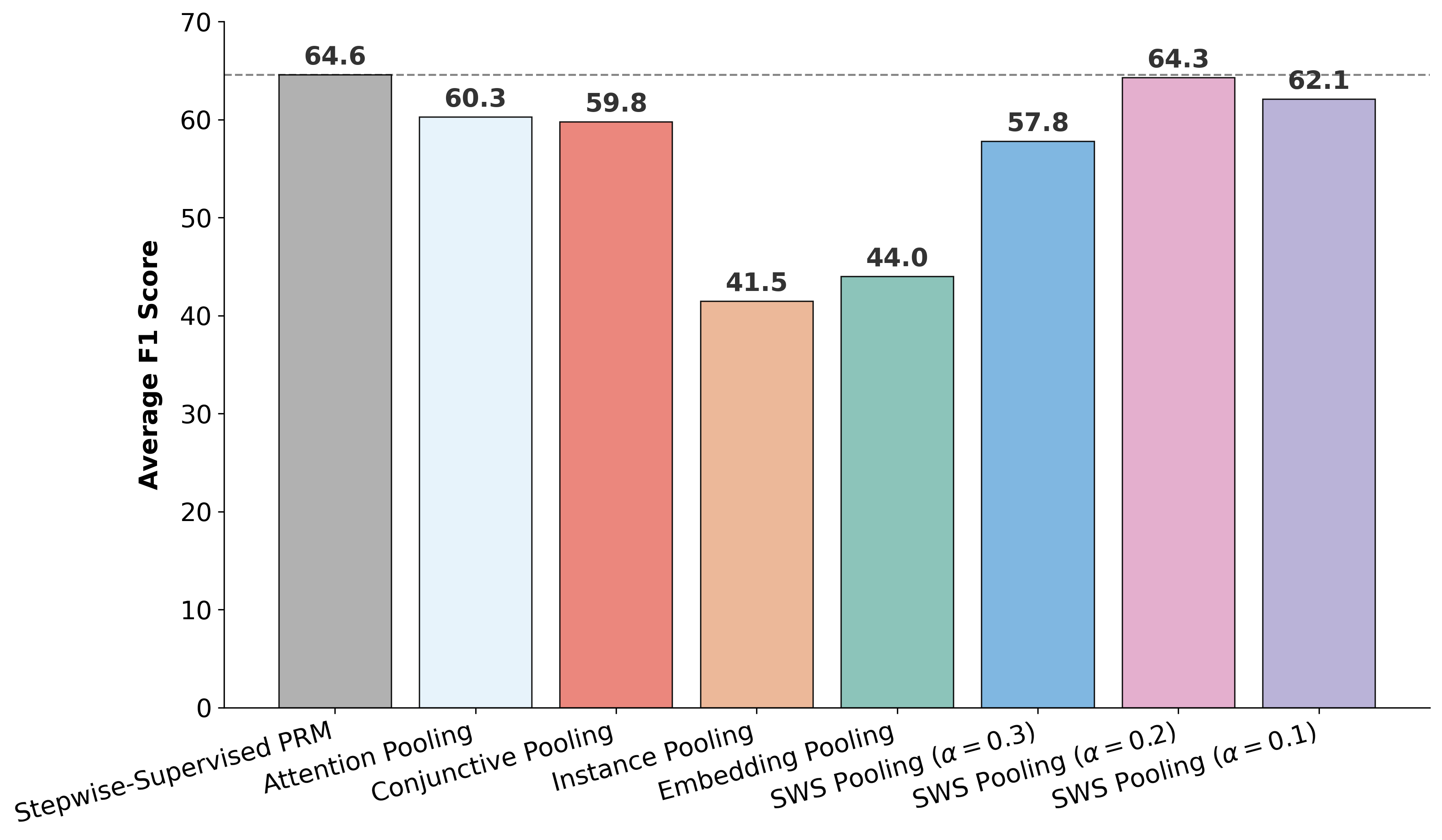}
  \caption{Average F1 scores of different pooling techniques on ProcessBench.}
  \label{fig:avg_f1_comparison}
\end{figure}

\paragraph{Comparing SWS Pooling with Other Pooling Techniques (RQ4).}
We compare SWS pooling against commonly used MIL pooling functions: instance pooling~\cite{instancepool1}, embedding pooling~\cite{embeddingpool1}, attention pooling~\cite{attnpool1}, and conjunctive pooling~\cite{conjpool1}. 
Appendix~\ref{app:pooling} describes these techniques in detail. 
We also include a fully supervised PRM that uses all step-level ground-truth labels as an upper bound baseline.
To avoid potential label noise and ensure consistency with standard MIL settings, we curate a training set from the human-annotated PRM800K dataset~\cite{letsverifystepbystep} and aggregate step labels to obtain trajectory labels (see Appendix~\ref{app:ablation-details} for details). We use Qwen3-4B~\cite{yang2025qwen3technicalreport} as the backbone to train PRMs.

Figure~\ref{fig:avg_f1_comparison} reports the average F1 scores on ProcessBench. SWS pooling with temperature $\alpha=0.2$ achieves the best score, surpassing other pooling techniques and closely matching the stepwise-supervised baseline. Temperatures higher or lower than 0.2 yield suboptimal performance, while the optimal temperature depends on data and model.

\section{Conclusion}

In this paper, we present a systematic analysis of outcome-supervised process reward modeling from the perspective of credit assignment, identifying the fundamental limitation of existing uniform and causal assignment strategies. To address this, we propose \emph{Weakest Link Assignment} as a principled alternative and introduce \mname{}, an outcome-supervised PRM framework that jointly learns credit assignment and reward modeling. By formulating outcome-supervised PRM as a Multiple Instance Learning (MIL) problem and introducing Softmax-Weighted-Sum pooling, we resolve the chicken-and-egg
dilemma between credit assignment and reward modeling under strong inter-instance dependence and redundancy. Both theoretical analysis and extensive experiments demonstrate the superiority of \mname{}, offering a principled alternative for fine-grained process supervision in complex reasoning tasks.

\section*{Limitations}
\phantomsection
\label{sec:limitations}

Despite achieving state-of-the-art performance among outcome-supervised
process reward modeling methods, our approach still lags behind PRMs trained on
human-annotated step-level labels. This gap highlights room for further
improvement in weakly supervised reward modeling.

Furthermore, using final answer correctness as a proxy for trajectory-level labels
may introduce inherent noise. LLMs sometimes arrive at the correct answer via flawed reasoning, in which case the Standard Multi-Instance Assumption does not strictly hold. While our method proves effective empirically, integrating techniques for learning with noisy labels into the MIL framework is a promising direction for future work.

Beyond the above, integrating PRMs into reinforcement learning (RL) remains another
important open problem. PRM-guided RL is susceptible to reward hacking, making this
a non-trivial challenge. Existing explorations include reward design, optimization
algorithms, and co-evolution of the reward model with the policy model. Notably, our
method requires only outcome supervision to train PRMs and incurs no additional
Monte Carlo sampling overhead, rendering it particularly suitable for reward-policy
co-evolution. We see this as a valuable avenue for future investigation.

\section*{Ethical Considerations}
All experiments in this work were carried out on publicly available datasets, models, and benchmarks, following their respective licenses and usage terms. The study does not involve human or animal subjects, and we did not collect, use, or disclose any personally identifiable information.

In this work, Large Language Models (LLMs) were used for language polishing and coding assistance. Specifically, LLMs supported refining the clarity and grammar of the manuscript, improving stylistic quality, and suggesting code snippets or troubleshooting strategies. All content generated by LLMs was carefully reviewed and verified by the authors before inclusion.  
The research design, critical analyzes, and all final decisions were independently conducted by the authors. LLMs were not involved in generating new research ideas or conclusions.

\section*{Acknowledgments}
This work is supported by the National Natural Science Foundation of China (Grant No.62506010). We thank the anonymous reviewers for their helpful feedback.


\bibliography{custom}

\clearpage
\appendix
\section{Related Work}\label{app:related-work}
\subsection{Multiple Instance Learning}

Multiple Instance Learning (MIL) is a weakly supervised learning paradigm in which labels are assigned to bags of instances rather than to individual instances. The task is to predict instance labels from bag labels alone. MIL has been applied to medical image analysis, text processing, time-series analysis, and video anomaly detection~\cite{milforimage, milfortext, milfortimeseries, milforvideo, conjpool1}. In this work, we formulate outcome-supervised process reward modeling explicitly as an MIL problem.

In the deep learning era, most MIL methods are pooling-based. Proposed pooling techniques include instance pooling~\cite{instancepool1}, embedding pooling~\cite{embeddingpool1}, attention pooling~\cite{addictivepool1,attnpool2}, additive pooling~\cite{addictivepool1}, max pooling~\cite{maxpool1}, and conjunctive pooling~\cite{milfortimeseries,milfortext,conjpool1}. Softmax-based pooling has received comparatively limited attention in
the MIL literature; one early example is a study on sound event detection~\cite{Salamon2017MIL}, which employs a non-parametric softmax and has seen little subsequent adoption in the MIL literature. In this work, we propose Softmax-Weighted-Sum (SWS) pooling, parameterized by a temperature $\alpha$ and tailored for bags with interdependent instances. Our theoretical analysis and experiments both demonstrate the crucial role of this hyperparameter for instance-level learnability.

\subsection{Process Reward Models}
Reward models are essential for improving large language models (LLMs) in complex reasoning tasks such as mathematical problem-solving~\cite{grattafiori2024llama3herdmodels,yang2024qwen25mathtechnicalreportmathematical} and competitive programming~\cite{hui2024qwen25codertechnicalreport}. Unlike outcome reward models (ORMs), which evaluate only the final answer, process reward models (PRMs)~\cite{uesato2022solvingmathwordproblems, letsverifystepbystep} assess each reasoning step, providing finer-grained feedback. Recent advancements have demonstrated the significant potential of PRMs in scaling test-time compute~\cite{rewardingprogress,zhang-etal-2025-lessons} and post-training~\cite{rstarmath,cheng2026stop}.

The main difficulty in training a PRM is obtaining reliable step-level annotations. Human annotation~\cite{letsverifystepbystep} is the primary source but is expensive and does not scale.  Monte Carlo methods~\cite{mathshepherd,luo2024improvemathematicalreasoninglanguage,rstarmath,ding2026scan} offer a promising alternative, estimating a step's quality by the empirical probability of reaching a correct final answer from that step. We refer to this as \emph{causal credit assignment}: a step's score reflects the expected success of its downstream completions rather than its own logical validity. 
This coupling causes errors from later steps to be incorrectly attributed to earlier ones, making it difficult to localize the true error source. 
Moreover, the estimates depend crucially on the choice of rollout policy and sampling strategy, which introduces noise and hampers generalization.
Another line of research~\cite{implicitprm,cui2025processreinforcementimplicitrewards} adopts DPO-style implicit reward functions~\cite{dpo}, which are trained as ORMs but used as PRMs. Because the trajectory-level score decomposes into a sum over steps, this approach implicitly enforces \emph{uniform credit assignment}: every step within the same trajectory receives the same optimization pressure, regardless of its actual contribution to the outcome. 
As a consequence, the model cannot differentiate genuinely erroneous steps from correct precursor steps, limiting its capacity for precise error identification. Moreover, these implicit reward functions are known to suffer from limited generalization~\cite{razin2026why}.

In this work, we advocate \emph{Weakest Link Assignment} as a principled alternative, drawing on the long-standing intuition: a reasoning chain is correct if and only if every step is correct, and incorrect if and only if any step is wrong.
This principle already underlies standard PRM evaluation and downstream usage:
ProcessBench~\cite{zheng-etal-2025-processbench} evaluates PRMs via first-error detection; best-of-N sampling~\cite{letsverifystepbystep} typically aggregates step scores by taking the minimum (for scalar rewards) or the product (for probabilities); and recent work has shown that a min-form reward aggregation in reinforcement learning substantially reduces reward hacking~\cite{cheng2026stop}. 
Building on this, we propose an outcome-supervised process reward modeling framework that jointly learns credit assignment and reward modeling under the Weakest Link principle, thereby closing the training–test gap and producing a PRM that accurately
attributes credit to the outcome-determining steps.

\section{Pooling Techniques in Multiple Instance Learning}
\label{app:pooling}

This appendix provides a concise overview of the pooling techniques mentioned in this paper. We assume a standard deep MIL architecture: given a bag $B = \{x_1, \dots, x_n\}$, a feature extractor $\phi$ (e.g., a neural network backbone) maps each instance to a $d$-dimensional embedding $z_i = \phi(x_i) \in \mathbb{R}^d$. A classification head $\psi$ then produces instance-level predictions $p_i \in [0,1]$. The pooling function aggregates the instance-level information to yield a bag-level prediction $\hat{y}(B)$.

We describe six pooling methods commonly used in the MIL literature.

\subsection{Embedding Pooling}

Embedding pooling aggregates instance embeddings before classification. The bag embedding is obtained by averaging the instance embeddings:
\[
z_{\text{bag}} = \frac{1}{n} \sum_{i=1}^n z_i,
\]
and the bag-level prediction is then $\hat{y}(B) = \psi(z_{\text{bag}})$. 
This approach is simple and widely used for bag-level classification, but it does not produce instance-level predictions and can lose discriminative instance information. 

\subsection{Instance Pooling}

Instance pooling performs classification on each instance independently, then aggregates the resulting instance predictions. The bag-level prediction is the average of the instance probabilities:
\[
p_i = \psi(z_i), \qquad \hat{y}(B) = \frac{1}{n} \sum_{i=1}^n p_i.
\]
Instance pooling naturally yields instance-level predictions, but the uniform averaging does not discriminate between more and less relevant instances.

\subsection{Max Pooling}

Max pooling selects the maximum instance prediction as the bag-level prediction:
\[
\hat{y}(B) = \max\{p_1, \dots, p_n\}.
\]
This directly implements the standard multi-instance assumption: a bag is positive if at least one instance is positive. However, it assigns all credit to a single instance, which can lead to unstable training and poor generalization.

\subsection{Attention Pooling}

Attention pooling learns instance-specific weights via an attention mechanism, then aggregates the instance embeddings accordingly. An attention head $\omega$ computes normalized attention weights:
\[
a_i = \frac{\exp(\omega(z_i))}{\sum_{j=1}^n \exp(\omega(z_j))},
\]
and the bag-level prediction is $\hat{y}(B) = \psi\left( \sum_{i=1}^n a_i z_i \right)$.
This method allows learnable, instance-specific weighting, but the attention weights are not directly constrained by the standard multi-instance  assumption, which can lead to degenerate solutions.

\subsection{Additive Pooling}

Additive pooling combines attention pooling and instance pooling. Instance embeddings are first scaled by their attention weights before classification:
\[
a_i = \frac{\exp(\omega(z_i))}{\sum_{j=1}^n \exp(\omega(z_j))}, \qquad
p_i = \psi(a_i z_i).
\]
The bag-level prediction is then:
\[
\hat{y}(B) = \frac{1}{n} \sum_{i=1}^n p_i.
\]
This hybrid approach encourages the model to attend to relevant instances while still performing instance-level classification.

\subsection{Conjunctive Pooling}

Conjunctive pooling computes attention weights and instance predictions independently, then combines them multiplicatively:
\[
a_i = \frac{\exp(\omega(z_i))}{\sum_{j=1}^n \exp(\omega(z_j))}, \qquad
p_i = \psi(z_i).
\]
The bag-level prediction is:
\[
\hat{y}(B) = \frac{1}{n} \sum_{i=1}^n a_i p_i.
\]
By training the attention and classification heads in parallel, this method avoids the potential degeneracy where the classifier relies on the attention head to pre-process the input, making the model more robust.

\subsection{Discussion}

Despite the empirical success of MIL across diverse applications, most common pooling techniques lack theoretical guarantees for instance-level classification accuracy. Attention-based pooling, currently the most widely used approach, has been shown not to be PAC-learnable~\cite{jang2024milpaclearnable}. Average-pooling methods do not align with the multi-instance assumption. Max-pooling, while most consistent with the assumption, suffers from low sample efficiency and training instability, which limits its practical adoption.

In contrast, our proposed temperature-parameterized SWS pooling exhibits a favorable property—instance-level Bayes consistency of the learning objective—under mild assumptions, and performs well in outcome-supervised PRM tasks. However, we note that the bag sizes in our setting are modest (typically under a few dozen), whereas some MIL applications involve bags with hundreds to thousands of instances~\cite{zheng2024dynamicpolicydrivenadaptivemultiinstance}. Whether SWS pooling remains effective in such large-bag regimes is uncertain and lies outside the scope of this work.

\section{Dataset and Implementation Details}\label{app:implement-details}

\subsection{Dataset Details}
\label{app:datasets}

We provide details of the datasets and benchmarks used in this work.

\begin{itemize}[leftmargin=*,itemsep=2pt,parsep=0pt,topsep=2pt,partopsep=0pt]

\item \textbf{Math-Shepherd}~\cite{mathshepherd}\footnote{\href{https://huggingface.co/datasets/peiyi9979/Math-Shepherd}{https://huggingface.co/datasets/peiyi9979/Math-Shepherd}} serves as the training set in our main experiments (Sections~\ref{sec:rq1} and~\ref{sec:rq2}).
It contains 445k English math reasoning trajectories with their associated problems.
Each reasoning step is labeled via Monte Carlo (MC) estimation without human annotation;
the MC rollouts are not included.
We use the entire dataset for training and do not perform evaluation on it.

\item \textbf{PRM800K}~\cite{letsverifystepbystep}\footnote{\href{https://github.com/openai/prm800k}{https://github.com/openai/prm800k}} is used as the training set for the ablation study on pooling techniques (Section~\ref{sec:rq34}).
The dataset comprises roughly 40k English math reasoning trajectories and their problems.
Human annotators label the location of the first error in each trajectory.
We adopt the original train/test split, yielding 37.5k training and 3.7k test trajectories.
The training split is further processed as described in Appendix~\ref{app:ablation-details};
the test split is not used for evaluation.

\item \textbf{MATH-500}~\cite{letsverifystepbystep}\footnote{\href{https://huggingface.co/datasets/HuggingFaceH4/MATH-500}{https://huggingface.co/datasets/HuggingFaceH4/MATH-500}} provides the problem set for best-of-N evaluation (Section~\ref{sec:rq2}).
It is a subset of the MATH dataset~\cite{MATH} comprising 500 English math problems.
We use all 500 problems as the test set.

\item \textbf{ProcessBench}~\cite{zheng-etal-2025-processbench}\footnote{\href{https://huggingface.co/datasets/Qwen/ProcessBench}{https://huggingface.co/datasets/Qwen/ProcessBench}} is the evaluation benchmark for our main experiments (Section~\ref{sec:rq1}).
It consists of four splits—GSM8K (0.4k), MATH (1k), OlympiadBench (1k), and Omni-MATH (1k)—which differ in difficulty.
Each problem comes with a reasoning chain, and all text is in English.
Every trajectory is manually annotated with the first error location.
We evaluate on all four splits. Evaluation details are given in Appendix~\ref{app:evaluation-details}.

\end{itemize}

\subsection{PRM Training Details}
\label{app:training-details}
\paragraph{Step delimiters and input formatting.}
For \mname{} and all trained baselines except ImplicitPRM~\cite{implicitprm}, we use the
double-newline sequence
\texttt{\textbackslash n\textbackslash n}
as the delimiter between the prompt and successive reasoning steps.
Accordingly, step-level representations or predictions are extracted at
the delimiter-ending positions described in Section~\ref{sec:method}.

ImplicitPRM is treated differently because it adopts a
DPO-style formulation whose implicit rewards are defined through sequence
log-probability ratios between the policy and reference models. Explicitly
inserting an additional step delimiter would change the token sequence and
therefore the generation probabilities on which its implicit rewards
depend. We thus do not insert the double-newline delimiter for
ImplicitPRM. Instead, each prompt--response pair is serialized using the
native chat template of the corresponding backbone, with identical
formatting applied to the policy and reference models.

\paragraph{Model architecture.}
For \mname{}, MathShepherd~\cite{mathshepherd}, OmegaPRM~\cite{luo2024improvemathematicalreasoninglanguage}, and SCAN~\cite{ding2026scan}, we remove the language modeling head from the LLM backbone and initialize a randomly initialized classification head, implemented as a two-layer MLP with the same hidden dimension as the LLM backbone. 
ImplicitPRM~\cite{implicitprm} uses the log-probability ratio between the policy model and the reference model (the same LLM backbone without fine-tuning) as an implicit reward model. 
For PQM~\cite{pqm}, we replace the language modeling head with a linear value head following the original paper.

\paragraph{Hyperparameters.}
For Qwen2.5-Math-7B-Instruct and Llama-3.2-3B-Instruct backbones, we set the learning rate to $1\times10^{-6}$; for the Qwen3-4B backbone, we use $1\times10^{-5}$. 
All experiments employ a linear learning rate scheduler. 
For all methods except PQM, we use a global batch size (GPU count × gradient accumulation steps × per-GPU batch size) of 384 on Qwen2.5-Math-7B-Instruct and llama backbones, and 512 on Qwen3-4B. 
For PQM, we adopt the more fine-grained batching strategy described in its original paper~\cite{pqm}. 
All models are trained for a single epoch.

\paragraph{Infrastructure.}
All experiments are conducted on four 80GB A800 GPUs using DeepSpeed ZeRO-3 parallel training~\cite{deepspeed}.

\subsection{Baseline Implementation Details}
\label{app:baseline-details}
MathShepherd, PQM, and ImplicitPRM are trained directly on the Math-Shepherd dataset. For ImplicitPRM, we use cross-entropy loss with $\beta = 0.05$. For PQM, we set $\zeta = 2$. Other implementation details follow the original papers.

OmegaPRM and SCAN improve the data collection and labeling strategy on top of MathShepherd and thus cannot be trained directly on the Math-Shepherd dataset. We therefore generate training data following their respective proposals. To ensure a fair comparison, we use the same set of problems as in Math-Shepherd dataset and  adopt MetaMath-Llemma-7B~\cite{yu2024metamath} as the
rollout model. For both methods, the amount of generated data is comparable to that of Math-Shepherd dataset.

For OmegaPRM, as no official code is available, we use the open-source implementation openr\footnote{\href{https://github.com/openreasoner/openr}{https://github.com/openreasoner/openr}}~\cite{wang2024openropensourceframework} for MCTS-based data generation. We set the search limit to 200 per question, $\alpha = 0.5$, $\beta = 0.9$, $L = 500$, and $c_{\text{puct}} = 0.125$. For each Monte Carlo estimation, we sample $k = 16$ rollouts. See the original paper for details~\cite{luo2024improvemathematicalreasoninglanguage}.

For SCAN, we follow the proposed denoising procedure on top of the Math-Shepherd dataset, with a tolerant distance $d = 2$. For each problem, we first generate $k = 16$ rollouts to estimate the confidence of the rollout model. For each trajectory in the dataset, we locate the first error (with label $=0$), then consider the first and second steps before that error, generate $k = 8$ completions for each, compute the Monte Carlo value, and combine it with the problem-level confidence to produce denoised labels. See the original paper for details~\cite{ding2026scan}.

\subsection{Evaluation Details}
\label{app:evaluation-details}
\paragraph{ProcessBench evaluation.}
We follow the evaluation protocol of ProcessBench~\cite{zheng-etal-2025-processbench}. For \mname{}, MathShepherd, and SCAN, PRMs are trained as standard binary classifiers; we therefore extract the earliest erroneous step directly from their per-step correctness predictions for reasoning steps. 
OmegaPRM, ImplicitPRM, and PQM produce scalar scores for each step; we first transform these scores into binary correctness predictions using a threshold. The threshold is chosen as the one that yields the highest F1 score on the MATH subset (see Table~\ref{tab:threshold}).

\begin{table}[t]
\centering
\small
\begin{tabular}{lcc}
\toprule
Method & \makecell{Qwen2.5-Math\\7B-Instruct} &
\makecell{Llama-3.2\\3B-Instruct} \\
\midrule
ImplicitPRM & 0.5 & 0.4 \\
OmegaPRM & 0.4 & 0.2 \\
PQM & 0.2 & 0.2 \\
\bottomrule
\end{tabular}
\caption{Thresholds for evaluating real-valued PRMs on ProcessBench across different backbones.}
\label{tab:threshold}
\end{table}

\paragraph{Best-of-N evaluation.}
We use Mistral-7B-Instruct-v0.2, Llama-3.1-8B-Instruct, and Qwen2.5-7B-Instruct as the generator models to produce evaluation data for best-of-N sampling on the MATH-500 problem set. The data for Mistral-7B-Instruct-v0.2 and Llama-3.1-8B-Instruct are taken directly from prior work~\cite{implicitprm}. For Qwen2.5-7B-Instruct, we collect 64 rollouts per problem in MATH-500 using temperature $0.7$, top-$p$ $0.95$, and a maximum generation length of $1024$. For binary classification models (\mname{}, MathShepherd, SCAN), we take the product of the step-wise scores as the trajectory score, consistent with earlier studies~\cite{letsverifystepbystep,zhang-etal-2025-lessons}. For models that output real-valued scores (OmegaPRM, ImplicitPRM, PQM), we take the minimum of the step-wise scores, following their original papers.

\subsection{Detailed Settings of Ablation on Pooling Techniques}
\label{app:ablation-details}
We implement embedding pooling, instance pooling, attention pooling, and 
conjunctive pooling following the descriptions in Appendix~\ref{app:pooling}. 
For attention and conjunctive pooling, the attention head is realized as a 
simple linear value head, as the attention‑based LLM backbone already provides sufficient nonlinearity.

As noted in \nameref{sec:limitations}, using final answer correctness as 
bag‑level labels may introduce label noise when correct answers are reached 
through flawed reasoning, violating the Standard Multi‑Instance Assumption. 
To eliminate this confounding factor and adhere strictly to the standard MIL 
setting, we construct a training set from the human‑annotated PRM800K dataset~\cite{letsverifystepbystep}. This dataset marks the first error location for each 
trajectory ($-1$ for error‑free trajectories). We derive trajectory‑level labels by checking whether the trajectory contains any error, and then discard the 
original step‑level annotations. The resulting dataset is further balanced to 
ensure equal proportions of erroneous and correct trajectories, yielding a 
final set of 30k trajectories (50\% containing errors), which we refer to as 
\textsc{PRM800K-balanced}. We then train PRMs on this dataset with different 
pooling techniques, using Qwen3‑4B as the backbone. Training details are 
provided in Appendix~\ref{app:training-details}.

For the stepwise-supervised baseline, we retain the step-level labels from \textsc{PRM800K-balanced}. Steps before the annotated first error are labeled 
correct; all others are labeled incorrect. The PRM is trained with step-level cross-entropy loss, with all other settings unchanged.

\section{Justification for Stop-Gradient Optimization}
\label{app:stop-gradient}

\subsection{Relationship to the Theoretical Analysis}

Theorems~\ref{thm:mil-minimize} and~\ref{thm:mil-converge}
characterize the population-level SWS risk and its Bayes-consistent
global minimizer. In practical optimization, we detach the SWS weights
from the computational graph. Specifically, the forward prediction is
implemented as
\begin{equation}
    \hat{p}_{\mathrm{sg}}
    =
    \sum_{t=1}^{T} \operatorname{sg}(w_t)p_t,
\end{equation}
where $\operatorname{sg}(\cdot)$ is the stop-gradient operator: it is
the identity in the forward pass and has zero derivative in the backward
pass. Therefore,
\begin{equation}
    \hat{p}_{\mathrm{sg}}
    =
    \sum_{t=1}^{T} w_t p_t
    =
    \hat{p}
\end{equation}
for every prediction vector $\bm{p}$. Stop-gradient thus leaves the
forward SWS prediction, the value of the population risk, and its set of
global minimizers unchanged. Consequently, the Bayes-consistency results (Theorems~\ref{thm:mil-minimize} and~\ref{thm:mil-converge})
continue to characterize the same forward risk.

However, stop-gradient changes the update used to optimize the
finite-sample empirical objective. The theoretical results characterize
the global minimizer of the population risk; they do not establish that
the practical optimization procedure necessarily converges to that
global minimizer. We therefore distinguish the statistical properties of
the forward risk from the optimization behavior of the practical
training algorithm.

\subsection{Effect on Optimization}

Without stop-gradient, differentiating through both the instance
predictions and their softmax weights gives
\begin{equation}
\frac{\partial \hat{p}}{\partial p_i}
=
w_i
+
\sum_{t=1}^{T}p_t\frac{\partial w_t}{\partial p_i}
=
\frac{w_i}{\alpha}
\left(\alpha+p_i-\hat{p}\right).
\end{equation}
When the weights are detached, the derivative becomes
\begin{equation}
\left.
\frac{\partial \hat{p}}{\partial p_i}
\right|_{\mathrm{sg}}
=
w_i.
\end{equation}
The difference between the two derivatives is
\begin{equation}
\frac{\partial \hat{p}}{\partial p_i}
-
\left.
\frac{\partial \hat{p}}{\partial p_i}
\right|_{\mathrm{sg}}
=
\frac{w_i}{\alpha}
\left(p_i-\hat{p}\right).
\end{equation}

For an instance whose current prediction is larger than the bag
prediction, i.e., $p_i>\hat{p}$, this additional term is positive; for an
instance with $p_i<\hat{p}$, it is negative. Consider an erroneous
trajectory, for which training encourages increasing $\hat{p}$. The full
gradient gives an even stronger update to prefixes that already have
large error probabilities, while weakening, or potentially reversing,
the update to prefixes with currently small predictions. Thus, an early
and possibly incorrect credit assignment can reinforce itself during
training. This effect becomes more pronounced at small $\alpha$ because
the additional term is scaled by $1/\alpha$. \textbf{Detaching the weights removes this self-reinforcing term while preserving
adaptive credit assignment in the forward pass, thereby helping
mitigate the risk of optimization getting trapped in a poor local optimum.}

\subsection{Stop-Gradient Ablation}
\begin{table}[t]
\centering
\small
\resizebox{\columnwidth}{!}{%
\begin{tabular}{@{}lccccc@{}}
\toprule
\textbf{Variant}
& \textbf{GSM8K}
& \textbf{MATH}
& \textbf{OlympiadBench}
& \textbf{Omni-MATH}
& \textbf{Avg.} \\
\midrule
\mname{} w/o SG
& 75.59 & 69.06 & 54.09 & 51.09 & 62.46 \\
\mname{}
& \textbf{78.08}
& \textbf{69.22}
& \textbf{57.19}
& \textbf{52.76}
& \textbf{64.31} \\
\bottomrule
\end{tabular}%
}
\caption{Stop-gradient ablation on ProcessBench. We compare \mname{}
with \mname{} without stop-gradient (w/o SG). Both variants use $\alpha=0.2$, with all other training settings being identical.}
\label{tab:stop-gradient-ablation}
\end{table}

To verify the above analysis, we conduct a controlled ablation following
the experimental setting of RQ4 (Section~\ref{sec:rq34}). Both variants
use Qwen3-4B as the backbone, \textsc{PRM800K-balanced} as the training
dataset, and an SWS temperature of $\alpha=0.2$. We compare \mname{}
with \mname{} without stop-gradient; the two variants differ only in
whether gradients are propagated through the SWS weights. We report
their ProcessBench performance in
Table~\ref{tab:stop-gradient-ablation}.

Applying stop-gradient improves F1 on all four ProcessBench subsets and
increases the average F1 from 62.46 to 64.31, an absolute improvement
of 1.85 points. These results support the practical benefit of removing
the self-reinforcing gradient term in this experimental setting.

\section{Additional Results}
\label{app:additional-results}

\begin{table*}[t]
\centering
\resizebox{\textwidth}{!}{
\begin{tabular}{lccccccccccccc}
\toprule
\multirow{2}[2]{*}{\textbf{Model}} & \multicolumn{3}{c}{\textbf{GSM8K}} & \multicolumn{3}{c}{\textbf{MATH}} & \multicolumn{3}{c}{\textbf{OlympiadBench}} & \multicolumn{3}{c}{\textbf{Omni-MATH}} & \multirow{2}[2]{*}{\begin{tabular}[c]{@{}c@{}} \bf Avg. \\ \bf F1 \end{tabular}} \\
\cmidrule(lr){2-4} \cmidrule(lr){5-7} \cmidrule(lr){8-10} \cmidrule(lr){11-13}
& Err. & Corr. & \textbf{F1} & Err. & Corr. & \textbf{F1} & Err. & Corr. & \textbf{F1} & Err. & Corr. & \textbf{F1} \\
\midrule
\multicolumn{14}{c}{\textit{Large Language Models as Critic}} \\
\midrule
Llama-3-8B-Instruct           & 42.5 & 7.8   & 13.1 & 28.6 & 9.1  & 13.8 & 27.1 & 2.7  & 4.8  & 26.1 & 8.3  & 12.6 & 11.1 \\
Llama-3.1-8B-Instruct         & 44.4 & 6.2   & 10.9 & 41.9 & 2.7  & 5.1  & 32.4 & 1.5  & 2.8  & 32.0 & 0.8  & 1.6  & 5.1  \\
Qwen2.5-7B-Instruct           & 40.6 & 33.2  & 36.5 & 30.8 & 45.1 & 36.6 & 26.5 & 33.9 & 29.7 & 26.2 & 28.6 & 27.4 & 32.6 \\
Qwen2.5-Math-7B-Instruct      & 15.5 & 100.0 & 26.8 & 14.8 & 96.8 & 25.7 & 7.7  & 91.7 & 14.2 & 6.9  & 88.0 & 12.7 & 19.9 \\
Qwen2.5-Coder-7B-Instruct     & 7.7  & 100.0 & 14.3 & 3.4  & 98.3 & 6.5  & 2.1  & 99.1 & 4.1  & 0.9  & 98.3 & 1.8  & 6.7  \\
\midrule
\multicolumn{14}{c}{\textit{Open Source Process Reward Models}} \\
\midrule
Math-Shepherd-PRM-7B          & 32.4 & 91.7  & 47.9 & 18.0 & 82.0 & 29.5 & 15.0 & 71.1 & 24.8 & 14.2 & 73.0 & 23.8 & 31.5 \\
RLHFlow-PRM-Mistral-8B        & 33.8 & 99.0  & 50.4 & 21.7 & 72.2 & 33.4 & 8.2  & 43.1 & 13.8 & 9.6  & 45.2 & 15.8 & 28.4 \\
RLHFlow-PRM-Deepseek-8B       & 24.2 & 98.4  & 38.8 & 21.4 & 80.0 & 33.8 & 10.1 & 51.0 & 16.9 & 10.9 & 51.9 & 16.9 & 26.6 \\
EurusPRM-Stage1               & 46.9 & 42.0  & 44.3 & 33.3 & 38.2 & 35.6 & 23.9 & 19.8 & 21.7 & 21.9 & 24.5 & 23.1 & 31.2 \\
EurusPRM-Stage2               & 51.2 & 44.0  & 47.3 & 36.4 & 35.0 & 35.7 & 25.7 & 18.0 & 21.2 & 23.1 & 19.1 & 20.9 & 31.3 \\
Qwen2.5-Math-7B-Math-Shepherd-PRM & 46.4 & 95.9  & 62.5 & 18.9 & 96.6 & 31.6 & 7.4  & 93.8 & 13.7 & 4.0  & 95.0 & 7.7  & 28.9 \\
\midrule
\multicolumn{14}{c}{\textit{Process Reward Models Trained on \textbf{Llama-3.2-3B-Instruct}}} \\
\midrule
ImplicitPRM & 41.1 & 82.4  & \textbf{54.8} & 22.9 & 66.5 & 34.1 & 13.9  & 55.2 & \underline{22.2} & 7.0  & 57.7 & 12.4  & \underline{30.9} \\
MathShepherd & 24.2 & 98.4  & 38.8 & 13.0 & 90.9 & 22.7 & 7.9  & 90.0 & 14.5 & 6.1  & 90.0 & 11.4  & 21.8 \\
SCAN & 31.4 & 97.9  & 47.6 & 16.3 & 88.9 & 27.6 & 8.9  & 87.0 & 16.2 & 6.6  & 86.3 & 12.2  & 25.9 \\
OmegaPRM & 37.2 & 40.4  & 38.7 & 21.9 & 48.0 & 30.1 & 10.9  & 54.9 & 18.2 & 7.9  & 63.5 & 14.1  & 25.3 \\
PQM & 21.3 & 90.7  & 34.4 & 26.3 & 70.7 & \underline{38.3} & 14.2  & 34.8 & 20.2 & 20.0  & 46.1 & \underline{27.9}  & 30.2 \\
\rowcolor{blue!10}
\mname{} (Ours, $\alpha=1$) & 35.7 & 75.1  & \underline{48.4} & 30.6 & 67.5 & \textbf{42.1} & 23.4  & 42.8 & \textbf{30.3} & 20.7 & 50.6  & \textbf{29.4} & \textbf{37.6} \\
\midrule
\multicolumn{14}{c}{\textit{Process Reward Models Trained on \textbf{Qwen2.5-Math-7B-Instruct}}} \\
\midrule
ImplicitPRM & 51.2 & 40.4  & 45.2 & 41.1 & 32.5 & 36.3 & 26.8  & 20.9 & 23.5 & 24.1  & 26.1 & \textbf{25.1}  & 32.5 \\
MathShepherd & 42.5 & 97.4  & 59.2 & 18.5 & 96.3 & 31.1 & 6.8  & 93.8 & 12.7 & 4.2  & 91.3 & 8.1  & 27.8 \\
SCAN & 95.3 & 44.4  & \underline{60.6} & 20.5 & 96.6 & 33.9 & 8.2  & 92.6 & 15.0 & 3.6  & 90.0 & 6.8  & 29.1 \\
OmegaPRM & 41.1 & 49.2  & 44.8 & 29.1 & 51.0 & 37.1 & 15.0  & 49.0 & 22.9 & 11.5  & 46.1 & 18.4  & 30.8 \\
PQM & 39.1 & 96.4  & 55.7 & 29.3 & 87.9 & \underline{43.9} & 17.7  & 73.2 & \textbf{28.5} & 14.8  & 71.4 & \underline{24.5}  & \underline{38.1} \\
\rowcolor{blue!10}
\mname{} (Ours, $\alpha=1$) & 51.2 & 91.7 & \textbf{65.7} & 33.4 & 90.6  & \textbf{48.7} & 15.4  & 86.1 & \underline{26.2} & 11.3  & 84.6 & 19.9  & \textbf{40.2} \\
\bottomrule
\end{tabular}
}
\caption{Full results of ProcessBench evaluation.}
\label{tab:all_processbench}
\end{table*}

\begin{figure}[t]
  \includegraphics[width=\columnwidth]{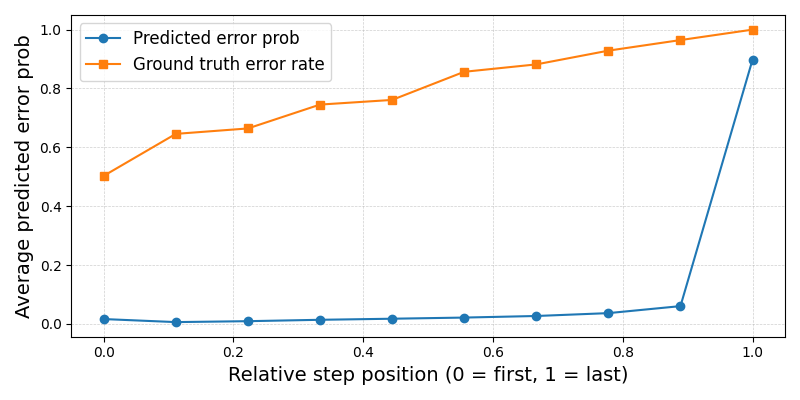}
  \caption{Average predicted error probability and ground‑truth error rate across relative step positions, computed over 1024 incorrect trajectories sampled from MATH split of Math-Shepherd dataset. The PRM is trained on Math‑Shepherd with max pooling for 1 epoch.}
  \label{fig:max_pool_trend}
\end{figure}

\begin{figure}[t]
  \includegraphics[width=\columnwidth]{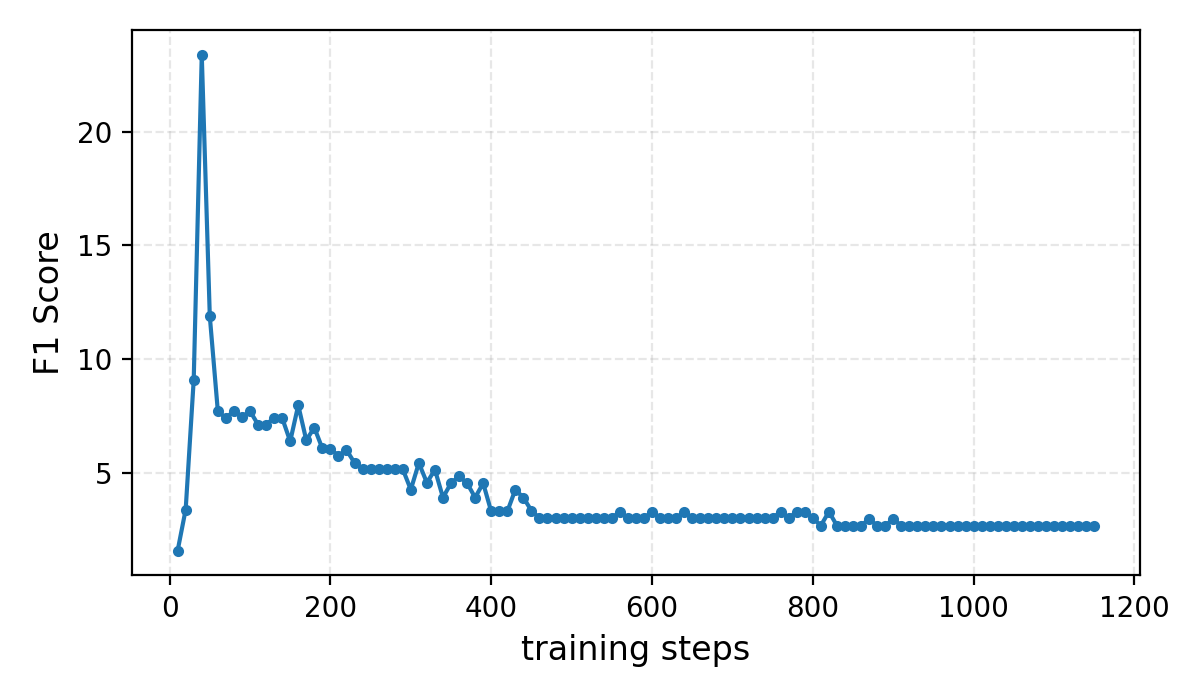}
  \caption{F1 score on MATH split of ProcessBench while training the PRM on Math‑Shepherd with max pooling.}
  \label{fig:max_pool_performance}
\end{figure}

\subsection{Preliminary Experiments}
\label{app:preliminary-experiments}
In early experiments, we applied max pooling as the MIL aggregation 
mechanism, since it directly instantiates the Standard Multi-Instance 
Assumption. Using Qwen2.5-Math-7B-Instruct as the backbone, we trained 
a PRM on the Math-Shepherd dataset. However, the PRM rapidly degenerated 
and performed poorly on downstream evaluations. 
Figure~\ref{fig:max_pool_trend} 
shows the average prediction values of the resulting PRM across different 
relative positions in erroneous reasoning trajectories (0 denotes the first step, 
1 the last step). The predictions remain near zero for all early positions but rise sharply at the final step. 
Figure~\ref{fig:max_pool_performance} shows that performance first improves and then collapses during training. This confirms our hypothesis that the strong dependence and redundancy among prefixes cause the classifier to quickly settle into degenerate solutions, failing to learn robust instance-level classification. In contrast, our proposed SWS pooling is provably Bayes consistent under the stated assumptions and thus rules out such degenerate global minimizers at the population level.

\subsection{Full Results of ProcessBench Evaluation}
\label{app:full-processbench-results}
We provide full results of ProcessBench evaluation in Table~\ref{tab:all_processbench}.

\subsection{Robustness to Label Noise}
\label{app:noise-robustness}

\begin{table}[t]
\centering
\resizebox{\columnwidth}{!}{%
\begin{tabular}{@{}lccccc@{}}
\toprule
\textbf{Noise Rate}
& \textbf{GSM8K}
& \textbf{MATH}
& \textbf{OlympiadBench}
& \textbf{Omni-MATH}
& \textbf{Avg.} \\
\midrule
0\%  & 78.08 & 69.22 & 57.19 & 52.76 & 64.31 \\
5\%  & 72.77 & 69.82 & 54.35 & 52.76 & 62.43 \\
10\% & 76.42 & 68.13 & 54.80 & 50.33 & 62.42 \\
15\% & 74.61 & 69.22 & 55.38 & 52.02 & 62.81 \\
20\% & 73.81 & 68.09 & 53.55 & 52.35 & 61.95 \\
\bottomrule
\end{tabular}%
}
\caption{F1 scores on ProcessBench under controlled one-sided label
noise.}
\label{tab:noise-robustness}
\end{table}

In our main experiments, we treat each reasoning trajectory as a bag of
prefixes and use final-answer correctness as its trajectory-level label.
This formulation assumes that the trajectory-level outcome label is
consistent with the latent process-level labels under the Standard
Multi-Instance Assumption: a trajectory is correct only if all its
reasoning steps are correct, whereas the presence of any erroneous step
makes the trajectory incorrect. However, as observed in recent work~\cite{zhang-etal-2025-lessons}, LLM-generated
trajectories may occasionally contain erroneous intermediate reasoning
despite reaching the correct final answer. A similar mismatch can also
occur in human-written reasoning. In such cases, an error-containing
trajectory is nevertheless labeled as correct according to its outcome,
so the Standard Multi-Instance Assumption no longer strictly holds and
the supervision contains \emph{one-sided label noise}.

To demonstrate the robustness of \mname{} to such one-sided noise, we conduct a controlled experiment that specifically
simulates this failure mode. Starting from \textsc{PRM800K-balanced}, we replace 5\%, 10\%,
15\%, or 20\% of the originally correct trajectories with trajectories
that contain reasoning errors while retaining their trajectory-level
labels as correct. The reported noise rate denotes the fraction of
originally correct trajectories that are replaced; trajectories originally labeled
as incorrect are left unchanged. All variants, including the 0\% clean control, use \mname{} with
Qwen3-4B as the backbone and $\alpha=0.2$, following the training
protocol described in Appendix~\ref{app:training-details}. We evaluate each resulting PRM on ProcessBench.

As shown in Table~\ref{tab:noise-robustness}, performance tends to
decrease as the injected noise rate increases. Even with 20\% injected
noise, \mname{} achieves an average F1 of 61.95, retaining approximately
96.3\% of its clean-data performance of 64.31. These results demonstrate
that \mname{} remains empirically robust under the controlled one-sided
label noise considered here.

\subsection{Generalization Beyond Mathematical Reasoning}
\label{app:non-math-generalization}
\begin{table}[t]
\centering
\small
\resizebox{\columnwidth}{!}{%
\begin{tabular}{@{}lccc@{}}
\toprule
\textbf{Method}
& \makecell{\textbf{Error}\\\textbf{Acc.}}
& \makecell{\textbf{Correct}\\\textbf{Acc.}}
& \textbf{F1} \\
\midrule
Step-supervised PRM
 & \textbf{77.09} & \textbf{36.40} & \textbf{49.45} \\
\mname{} ($\alpha=0.03$)
 & 75.60 & 33.07 & 46.01 \\
\bottomrule
\end{tabular}%
}
\caption{First-error detection results on the held-out test split of
MMLU-Pro-CoT. Error and Correct Acc. denote accuracies on trajectories
with and without reasoning errors, respectively.}
\label{tab:mmlu-pro-cot}
\end{table}

To examine whether \mname{} generalizes beyond mathematical reasoning,
we conduct an additional experiment on MMLU-Pro-CoT%
~\cite{zeng2025versaprm}, a multi-domain chain-of-thought dataset
derived from MMLU-Pro~\cite{wang2024mmlupro}. MMLU-Pro covers 14
domains, including physics, chemistry, engineering, law, psychology,
and health, and contains challenging questions that require explicit
reasoning. MMLU-Pro-CoT collects reasoning trajectories for these questions and provides step-level correctness annotations. We split the dataset randomly into 67,337 training,
8,401 development, and 8,360 test examples. 

For \mname{}, we aggregate the step labels into trajectory-level labels
in the same manner as in our construction of
\textsc{PRM800K-balanced}, and then discard all original step-level
annotations during training. A trajectory is labeled as incorrect if
it contains at least one erroneous step. We search for the optimal
temperature on the development set and select the best-performing value,
$\alpha=0.03$. As an empirical upper bound, we train a step-supervised
PRM using all ground-truth step-level annotations. Both models use
Qwen3-4B as the backbone and follow the training protocol described in
Appendix~\ref{app:training-details}. They are trained on the training
split and evaluated for first-error detection on the held-out test set.

As shown in Table~\ref{tab:mmlu-pro-cot}, \mname{} achieves an F1 score
of 46.01 without using any process-level supervision, reaching 93.0\%
of the fully step-supervised PRM's F1 score of 49.45, with an absolute
gap of 3.44 points. This result provides evidence that \mname{} remains
effective beyond mathematical reasoning across a diverse collection of
reasoning domains.

\subsection{Case Studies}
\label{app:case-study}
\begin{table*}[t]
\centering
\small
\caption{
Case study (1). 
The first error occurs in step 3, which uses repeated numbers (2, 2) despite the requirement for distinct integers.
Blue values indicate the predicted probability that the corresponding prefix is incorrect.
}
\label{tab:case-1}

\raggedright
\textbf{Question}

\begin{tcolorbox}[
colback=nmgray,
colframe=white,
width=\linewidth,
arc=0mm
]
The product of a set of distinct positive integers is 84. What is the least possible sum of these integers?
\newline(first wrong step: 3)
\end{tcolorbox}

\vspace{0.3em}

\raggedright
\textbf{Step 1}

\begin{tcolorbox}[
colback=cyan!10,
colframe=white,
width=\linewidth,
arc=0mm
]
To find the least possible sum of distinct positive integers whose product is 84, let's start by prime factorizing 84: \[84 = 2^2 \times 3 \times 7\] Our goal is to minimize the sum of these factors while ensuring they multiply to 84. To do this, we want to use as few numbers as possible because larger numbers generally lead to larger sums.
\newline\textcolor{blue}{(0.18)}
\end{tcolorbox}

\raggedright
\textbf{Step 2}

\begin{tcolorbox}[
colback=cyan!10,
colframe=white,
width=\linewidth,
arc=0mm
]
Start with the largest prime factor: The largest prime factor here is 7. If we include 7 in our set, we have \(84/7 = 12\). Now, we need to factorize 12: \(12 = 2^2 \times 3\).
\newline\textcolor{blue}{(0.03)}
\end{tcolorbox}

\raggedright
\textbf{Step 3}

\begin{tcolorbox}[
colback=pink!20,
colframe=white,
width=\linewidth,
arc=0mm
]
Include the necessary factors: We already have a \(2^2\) from our prime factorization. Since we're trying to minimize the sum, we should include the smallest possible numbers that multiply to the remaining factors. First, for \(2^2\), we already have one 2 (from \(2^2\)) included. Second, for the remaining \(2^1\) and \(3\), adding these directly gives us the necessary components. So, the set of integers that multiply to 84 and have the least possible sum would be \(7, 2, 2, 3\).
\newline \textcolor{blue}{(0.84)}
\end{tcolorbox}

\raggedright
\textbf{Step 4}

\begin{tcolorbox}[
colback=pink!20,
colframe=white,
width=\linewidth,
arc=0mm
]
Calculating the sum: \[7 + 2 + 2 + 3 = 14\] Therefore, the least possible sum of these integers is **14**.
\newline \textcolor{blue}{(0.94)}
\end{tcolorbox}

\end{table*}

\begin{table*}[t]
\centering
\small

\caption{
Case study (2). 
The first error occurs in step 3 when multiplying the numbers: 6×518400=3110400, not 3104400.
Blue values indicate the predicted probability that the corresponding prefix is incorrect.
}
\label{tab:case-2}

\raggedright
\textbf{Question}

\begin{tcolorbox}[
colback=nmgray,
colframe=white,
width=\linewidth,
arc=0mm
]
The Coventry School's European debate club has 6 German delegates, 5 French delegates, and 3 Italian delegates.  In how many ways can these 14 delegates sit in a row of 14 chairs, if each country's delegates insist on all sitting next to each other?
\newline(first wrong step: 3)
\end{tcolorbox}

\vspace{0.3em}

\raggedright
\textbf{Step 1}

\begin{tcolorbox}[
colback=cyan!10,
colframe=white,
width=\linewidth,
arc=0mm
]
To find the total number of ways the delegates can sit in a row of 14 chairs with each country's delegates sitting together, we can break down the problem into smaller steps. First, consider the three groups of delegates as single units. Since there are three countries, we have 3 units (groups) that need to be arranged in a row. The number of ways to arrange these 3 units is 3!. This is because for the first position, we have 3 choices, for the second position, we have 2 choices left, and for the last position, we have only 1 choice left. 3! = 3 * 2 * 1 = 6.
\newline\textcolor{blue}{(0.14)}
\end{tcolorbox}

\raggedright
\textbf{Step 2}

\begin{tcolorbox}[
colback=cyan!10,
colframe=white,
width=\linewidth,
arc=0mm
]
Next, let's consider the internal arrangement of each group. For the 6 German delegates, they can be arranged in 6! ways. For the 5 French delegates, they can be arranged in 5! ways. And for the 3 Italian delegates, they can be arranged in 3! ways. 6! = 6 * 5 * 4 * 3 * 2 * 1 = 720, 5! = 5 * 4 * 3 * 2 * 1 = 120, and 3! = 3 * 2 * 1 = 6.
\newline\textcolor{blue}{(0.07)}
\end{tcolorbox}

\raggedright
\textbf{Step 3}

\begin{tcolorbox}[
colback=pink!20,
colframe=white,
width=\linewidth,
arc=0mm
]
To find the total number of arrangements, we multiply the number of ways to arrange the units (Step 2) by the number of ways to arrange each group internally (Step 3). Total number of arrangements = (number of ways to arrange units) * (number of ways to arrange Germans) * (number of ways to arrange French) * (number of ways to arrange Italians) = 6 * 720 * 120 * 6 = 6 * 518400 = 3104400.
\newline \textcolor{blue}{(0.99)}
\end{tcolorbox}

\raggedright
\textbf{Step 4}

\begin{tcolorbox}[
colback=pink!20,
colframe=white,
width=\linewidth,
arc=0mm
]
Therefore, the final answer is: \boxed{3104400}.
\newline \textcolor{blue}{(0.99)}
\end{tcolorbox}

\end{table*}

We present qualitative examples from ProcessBench in Tables~\ref{tab:case-1} and~\ref{tab:case-2}. The blue values represent the predicted probability that the corresponding prefix is incorrect, as estimated by the PRM trained with \mname{}. The predicted probability becomes high once the first erroneous step is encountered.

\section{Extension to Reasoning with Reflection and Backtracking}
\label{app:reflection}

In the main text, we restrict our discussion to sequential reasoning without reflection or backtracking. 
Under this setting, any logical error, once introduced, is inherited by all subsequent steps and ultimately leads to an incorrect final answer. 
The Weakest Link principle therefore holds: a chain is correct iff every step is correct, and incorrect iff any step contains an error.
For reasoning that allows reflection and backtracking, however, a later step may correct an earlier mistake. 
Consequently, a trajectory that arrives at a correct final answer may temporarily contain intermediate errors, violating the Weakest Link principle in its strict sequential form. 

Below, we discuss how \mname{} could be conceptually extended to
reasoning with explicit reflection or backtracking. The central idea is
to evaluate only the currently active derivation after retracted content
has been removed. Specifically, by lifting the formulation from linear
\emph{prefixes} to graph-structured \emph{reasoning states}, sequential
reasoning and reasoning with reflection can be represented within a
unified framework.

\paragraph{Reasoning state.}
A \emph{reasoning state} is a directed acyclic graph (DAG)
\[
G = (\mathcal{V}, \mathcal{E}),
\]
where each node $v \in \mathcal{V}$ represents a premise, an intermediate conclusion, or the final answer, and each directed edge $(u \to v) \in \mathcal{E}$ indicates that $v$ is logically derived from $u$.

\paragraph{Correctness of reasoning states.}
A reasoning state $G$ is said to contain an error if there exists a node whose stated result is factually false or cannot be validly inferred from its parent nodes.
We define the correctness label of $G$ as
\[
y(G) \in \{0,1\},
\]
where $1$ denotes the presence of at least one logical error and $0$ denotes an error‑free state.
Under this definition, correctness is monotone with respect to graph extension:
\begin{equation}\label{eq:label-monotonic-graph}
G_1 \subseteq G_2 \;\Longrightarrow\; y(G_2) \ge y(G_1),
\end{equation}
since any error present in a subgraph $G_1$ is inherited by every supergraph $G_2$ that extends it.

\paragraph{Reasoning as graph construction.}
A reasoning process can be viewed as the step-by-step construction of a reasoning state, starting from the initial premises and adding, modifying, or deleting nodes and edges until a node representing the final answer is produced.

For sequential reasoning without reflection, each step adds new nodes and edges to the current graph without removing any existing ones.
Hence, every prefix of the trajectory corresponds to a subgraph of the subsequent prefix, and the monotonicity of prefix correctness (Eq.~\ref{eq:label-monotonic}) follows immediately as a special case of Eq.~\ref{eq:label-monotonic-graph}.

For reasoning with reflection or backtracking, a step may revise or retract earlier conclusions, i.e., it may delete nodes or edges.
In this case, earlier reasoning states are not necessarily subgraphs of later ones.
Nevertheless, the monotonicity relation $G_1 \subseteq G_2 \Rightarrow y(G_2) \ge y(G_1)$ still holds for any pair where $G_1$ is a subgraph of $G_2$.

\paragraph{Applying our PRM to reasoning with reflection.}
Our method, \mname{}, trains the PRM on sequential reasoning trajectories and thus learns a classifier $f$ that maps a reasoning state to its error probability:
\[
f(G) \in [0,1],
\]
with the objective $f(G) \approx y(G)$.
Given a trajectory with reflection $\tau = \langle q, s_1, \dots, s_T\rangle$, we first convert each prefix $\tau_{:t}$ into its corresponding reasoning state $G_t$ (the DAG constructed up to step $t$) and linearize $G_t$ into a sequential form, i.e., a prefix without backtracking.
We then evaluate $f(G_t)$ for each step.

The step‑wise quality of reasoning can be assessed by the change in the correctness of reasoning states:
\begin{itemize}[leftmargin=*,itemsep=2pt,parsep=0pt,topsep=2pt,partopsep=0pt]
    \item If $f(G_{t}) < f(G_{t-1})$, step $t$ improves the reasoning state (e.g., by correcting an earlier mistake).
    \item If $f(G_{t}) > f(G_{t-1})$, step $t$ introduces or amplifies an error.
    \item If $f(G_{t}) \approx f(G_{t-1})$, step $t$ is neutral.
\end{itemize}

Thus, although our PRM is trained solely on sequential trajectories
without reflection, it could in principle be applied to successive
active derivations and thereby provide fine-grained feedback on how each
reflection or backtracking operation changes the predicted error
probability.

\onecolumn
\section{Proof of Theorems}
\label{app:theorem}

\subsection{Formulations and Assumptions}
\begin{definition}[Instance space]
    Let $\mathcal{X}$ denote the \emph{instance space}. 
    Each instance $x \in \mathcal{X}$ is associated with a label
    \[
        y(x) \in \{0,1\},
    \] 
    which induces a partition of $\mathcal{X}$ into the \emph{positive instance set} and the \emph{negative instance set}:
    \[
        \mathcal{X}^+ := \{ x \in \mathcal{X} \mid y(x) = 1 \}, \qquad
        \mathcal{X}^- := \{ x \in \mathcal{X} \mid y(x) = 0 \}.
    \]
\end{definition}

\begin{definition}[Bag space]
    A \emph{bag} (or multi-instance) $B$ is defined as a finite set of instances from $\mathcal{X}$:
    \[
        B = (x_1, x_2, \dots, x_n), \quad x_i \in \mathcal{X}, \quad 1 \le n \le N,
    \]
    where $N$ is the maximum allowed bag size. The \emph{bag space} $\mathcal{B}$ is the set of all possible bags:
    \[
        \mathcal{B} = \bigcup_{n=1}^N \mathcal{X}^n.
    \]
    Each bag $B$ is associated with a label $y(B) \in \{0,1\}$, defined via the instance labels:
    \[
        y(B) = \mathbf{1}\Big( \exists i \in \{1,\dots,n\} \text{ s.t. } y(x_i)=1 \Big),
    \]
    i.e., a bag is labeled positive if it contains at least one positive instance, and negative otherwise.
\end{definition}

\begin{definition}[Bag distribution]
    Let $\mathcal{B}$ be the bag space as defined above. 
    A \emph{bag distribution} is a probability measure $\mathbb{P}$ on $\mathcal{B}$ that describes how bags are generated. 
    In general, the generation process can be described in two steps:
    \begin{enumerate}
        \item The bag size $n$ is drawn from a discrete distribution over $\{1,2,\dots,N\}$:
        \[
            n \sim P_N(n),
        \]
        where $N$ is the maximum allowed bag size.
        \item Conditional on the bag size $n$, the instances $(x_1, \dots, x_n)$ are jointly drawn according to some distribution $\mathbb{P}_n$ on $\mathcal{X}^n$:
        \[
            (x_1, \dots, x_n) \sim \mathbb{P}_n.
        \]
    \end{enumerate}
    Therefore, the overall distribution of a bag $B = (x_1, \dots, x_n)$ is
    \[
        \mathbb{P}(B) = P_N(n) \cdot \mathbb{P}_n(x_1, \dots, x_n),
    \]
    where no independence assumption on the instances is made; they may follow an arbitrary joint distribution.
\end{definition}

\begin{definition}[Bag cover set (Restated)]
    Let $S \subseteq \mathcal{X}$ be a set of instances. 
    The \emph{bag cover set} of $S$ is the set of all bags in $\mathcal{B}$ that contain at least one instance from $S$:
    \[
        \mathcal{C}(S) := \{ B \in \mathcal{B} \mid B \cap S \neq \emptyset \}.
    \]
    That is, a bag $B$ belongs to $\mathcal{C}(S)$ if and only if it contains at least one instance from $S$.

    Similarly, we define the \emph{positive bag cover set} and \emph{negative bag cover set} of $S$ as follows:
    \[
        \mathcal{C}^+(S) := \{ B \in \mathcal{B} \mid B \cap S \neq \emptyset \wedge y(B) = 1 \};
    \]
    \[
        \mathcal{C}^-(S) := \{ B \in \mathcal{B} \mid B \cap S \neq \emptyset \wedge y(B) = 0 \}.
    \]

    A bag cover set $\mathcal{C}(S)$ (or $\mathcal{C}^+(S), \mathcal{C}^-(S)$) is called \emph{non-trivial} if its probability under the bag distribution $\mathbb{P}$ is strictly positive:
    \[
        \mathbb{P}(\mathcal{C}(S)) > 0 \quad (\text{or } \mathbb{P}(\mathcal{C}^+(S)) > 0, \, \mathbb{P}(\mathcal{C}^-(S)) > 0).
    \]
\end{definition}

\begin{assumption}[$M$-Assumption (Restated)]\label{ass:M-restated}
    There exists a constant $M > 0$ such that for any set of negative instances $S \subseteq \mathcal{X}^-$, if
    \[
        \mathbb{P}(\mathcal{C}(S)) = \mathbb{P}(\mathcal{C}^+(S)) + \mathbb{P}(\mathcal{C}^-(S)) > 0,
    \]
    then
    \[
        \frac{\mathbb{P}(\mathcal{C}^+(S))}{\mathbb{P}(\mathcal{C}^-(S))} \le M,
    \]
    where $\mathbb{P}(\mathcal{C}^+(S))$ and $\mathbb{P}(\mathcal{C}^-(S))$ denote the probabilities of the positive and negative bag cover sets under the bag distribution, respectively.
\end{assumption}

\subsection{Mathematical Properties of Softmax-Weighted-Sum}
We first give a few lemmas about the mathematical properties of the Softmax-Weighted-Sum (SWS) function.

\begin{lemma}[Partial derivative of SWS]\label{lem:derivative}
Let the temperature parameter be $\alpha>0$ and let $\bm{p} = (p_1,\dots,p_n)\in[0,1]^n$.
Define the softmax weights $w_i$ and the softmax weighted sum $\hat{p}$ as
\[
\hat{p}(\bm{p})=\sum_{i=1}^n w_i p_i,\qquad
w_i=\frac{e^{p_i/\alpha}}{\sum_{j=1}^n e^{p_j/\alpha}}.
\]
Then for any $i\in\{1,\dots,n\}$ the partial derivative admits the representation
\[
\frac{\partial \hat{p}}{\partial p_i} = \frac{w_i}{\alpha}\bigl( \alpha + p_i - \hat{p} \bigr).
\]
In particular, the sign of the derivative is completely determined by the sign of $(\alpha + p_i - \hat{p})$:
\[
\operatorname{sgn}\left(\frac{\partial \hat{p}}{\partial p_i}\right)
= \operatorname{sgn}\bigl(\alpha + p_i - \hat{p}\bigr).
\]
\end{lemma}

\begin{lemma}[Minimum of SWS under $p_1 = \varepsilon$]\label{lem:minimum-base}
Let $\alpha>0$, $\varepsilon\in(0,1]$, and let $n\in\mathbb{N}$, $n\ge 1$.
For any $\bm{p}=(p_1,\dots,p_n)\in[0,1]^n$ with $p_1 = \varepsilon$, define the softmax weighted sum $\hat{p}(\bm{p})$ as in Lemma~\ref{lem:derivative}.
Then the minimum of $\hat{p}$ over all such $\bm{p}$ is
\[
\hat{p}_{\min}=
\begin{cases}
\dfrac{\varepsilon\, e^{\varepsilon/\alpha}}{e^{\varepsilon/\alpha}+n-1}, 
& \text{if } \varepsilon \le \alpha(u+1), \\[1.6em]
\varepsilon - \alpha\, u, 
& \text{if } \varepsilon > \alpha(u+1),
\end{cases}
\]
where $u = W_0\!\big(\frac{n-1}{e}\big)$ and $W_0$ is the principal branch of the Lambert $W$ function.
For $n=1$ the expression reduces to $\hat{p}_{\min}=\varepsilon$.
\end{lemma}

\begin{proof}
Fix $p_1=\varepsilon$ and write
\[
\hat p(\bm{p})=\frac{\varepsilon e^{\varepsilon/\alpha}+ \sum_{i=2}^n p_i e^{p_i/\alpha}}
                 {e^{\varepsilon/\alpha}+ \sum_{i=2}^n e^{p_i/\alpha}} .
\]
By Lemma~\ref{lem:derivative}, for every $i\ge 2$,
\[
\frac{\partial \hat p}{\partial p_i}= \frac{w_i}{\alpha}(\alpha+p_i-\hat p).
\]
Since $w_i>0$, a minimizer has each $p_i$ either satisfying $p_i=\hat p-\alpha$ or lying on the boundary $\{0,1\}$.  
If some $p_i=1$, then $\alpha+1-\hat p>0$, so the derivative is positive; decreasing that coordinate would lower $\hat p$, contradicting minimality. Hence every $p_i\;(i\ge2)$ is either $0$ or satisfies $p_i=\hat p-\alpha$.

Suppose a minimizer contains both $0$ and an interior point.
For $p_i=0$ the inward derivative must be non‑negative, giving $\hat p\le\alpha$.
For an interior point $p_j$ we have $p_j=\hat p-\alpha>0$, yielding $\hat p>\alpha$, a contradiction.
Therefore a minimizer consists either entirely of zeros or entirely of a common value $c\in(0,1)$ obeying $c=\hat p-\alpha$.

\noindent\textbf{Case 1:} $p_2=\dots=p_n=0$.
\[
\hat p_{\text{I}}=\frac{\varepsilon e^{\varepsilon/\alpha}}{e^{\varepsilon/\alpha}+n-1}.
\]

\noindent\textbf{Case 2:} $p_2=\dots=p_n=c$ with $c=\hat p-\alpha$.
Substituting into the definition of $\hat p$ and rearranging gives
\[
\frac{\varepsilon-c-\alpha}{\alpha}\, e^{(\varepsilon-c)/\alpha}=n-1.
\]
Let $u=W_0\!\bigl(\frac{n-1}{e}\bigr)$, i.e.\ $ue^u=\frac{n-1}{e}$. Then
\[
c=\varepsilon-\alpha(u+1),\qquad
\hat p_{\text{II}}=\varepsilon-\alpha u .
\]
Feasibility of the interior solution requires $c>0$, which is equivalent to
$\varepsilon>\alpha(u+1)$.

Now define the function
\[
\varphi(\varepsilon)=
\begin{cases}
\hat p_{\text{I}}, & \varepsilon\le\alpha(u+1),\\[2pt]
\hat p_{\text{II}}, & \varepsilon>\alpha(u+1).
\end{cases}
\]
Both expressions equal $\alpha$ at $\varepsilon=\alpha(u+1)$, and $\varphi$ is strictly increasing on each interval (the derivative of $\hat p_{\text{I}}$ is positive, while $\hat p_{\text{II}}$ has derivative $1$). Thus $\varphi$ is continuous and increasing on $(0,1]$. For $\varepsilon\le\alpha(u+1)$ the interior candidate is infeasible, while for $\varepsilon>\alpha(u+1)$ the all‑zero candidate violates the necessary boundary condition $\hat p\le\alpha$, hence cannot be a local minimum. Consequently the global minimum is precisely $\varphi(\varepsilon)$, which yields the claimed formula. The case $n=1$ gives $u=0$ and both branches reduce to $\hat p_{\min}=\varepsilon$.
\end{proof}

\begin{lemma}[Minimum of SWS under $\max_i p_i \ge \varepsilon$]\label{lem:minimum-advanced}
Let $\alpha>0$, $\varepsilon\in(0,1]$, and $n\in\mathbb{N}$, $n\ge 1$.  
For any $\bm{p}=(p_1,\dots,p_n)\in[0,1]^n$ with $\max_i p_i \ge \varepsilon$, define the softmax weighted sum $\hat{p}(\bm{p})$ as in Lemma~\ref{lem:derivative}.
Then the minimum of $\hat{p}$ over all such $\bm{p}$ is 
\[
\hat{p}_{\min}=
\begin{cases}
\dfrac{\varepsilon\, e^{\varepsilon/\alpha}}{e^{\varepsilon/\alpha}+n-1}, 
& \text{if } \varepsilon \le \alpha(u+1), \\[1.6em]
\varepsilon - \alpha\, u, 
& \text{if } \varepsilon > \alpha(u+1),
\end{cases}
\]
where $u = W_0\!\bigl(\frac{n-1}{e}\bigr)$ and $W_0$ is the principal branch of the Lambert $W$ function.  
For $n=1$ this reduces to $\hat{p}_{\min}=\varepsilon$.
\end{lemma}

\begin{proof}
Let's first consider the minimum of $\hat{p}$ under the constraint $p_1 \ge \varepsilon$, i.e., 
\[
\min_{\bm{p}\,:\,p_1 \ge \varepsilon} \hat{p}(\bm{p}).
\]
For any admissible $\bm{p}$, set $m = p_1 \ge \varepsilon$. Consider the family of vectors with a fixed value $m$ in the first coordinate and arbitrary remaining coordinates in $[0,1]$.  The minimum of $\hat{p}$ over this family is a function $\varphi(m)$, which by Lemma~\ref{lem:minimum-base} equals
\[
    \varphi(m)=\min_{\bm{p}\,:\,p_1 = m} \hat{p}(\bm{p})=
\begin{cases}
\dfrac{m\, e^{m/\alpha}}{e^{m/\alpha}+n-1}, & m \le \alpha(u+1), \\[1.2em]
m - \alpha\, u, & m > \alpha(u+1),
\end{cases}
\]
where $u = W_0\bigl(\frac{n-1}{e}\bigr)$. 

Now $\varphi(m)$ is strictly increasing for $m\in(0,1]$.  Indeed, for $m<\alpha(u+1)$,
\[
   w = \frac{e^{m/\alpha}}{e^{m/\alpha}+n-1}, \qquad
   \varphi'(m) = w\bigl[1 + \tfrac{m}{\alpha}(1-w)\bigr] > 0.
\]

while for $m>\alpha(u+1)$ we have $\varphi'(m)=1>0$; the two branches meet continuously at $m=\alpha(u+1)$ where both give $\varphi=\alpha$.  Hence $\varphi$ is non‑decreasing, and its minimum on the interval $[\varepsilon,1]$ is attained at the left endpoint $m=\varepsilon$.  Consequently, 
\[
\min_{\bm{p}\,:\,p_1 \ge \varepsilon} \hat{p}(\bm{p}) = \min_{m\ge \varepsilon} \varphi(m) = \varphi(\varepsilon).
\]

Since $\{\bm p:\max_i p_i\ge\varepsilon\}
=\bigcup_{j=1}^n\{\bm p:p_j\ge\varepsilon\}$ and $\hat p$ is invariant
under coordinate permutations, all $n$ constrained subproblems have the
same minimum. Therefore,
\[
\min_{\bm{p}\,:\,\max p_i \ge \varepsilon} \hat{p}(\bm{p}) = 
\min_{\bm{p}\,:\,p_1 \ge \varepsilon} \hat{p}(\bm{p}) = \varphi(\varepsilon).
\]
This completes the proof.
\end{proof}

\begin{lemma}[Increment lower bound for SWS]\label{lem:increment}
    Let $\alpha>0$, $\varepsilon\in(0,1]$, and $N\in\mathbb{N}$ ($N\ge 1$).  
    For any $\bm{p}=(p_1,\dots,p_n)\in[0,1]^n$ with $n\le N$ and $p_1 \le 1-\varepsilon$, fix $p_2,\dots,p_n$ and define $\bm{p'}=(1,p_2,\dots,p_n)$.
    Define the softmax weighted sum $\hat{p}(\cdot)$ as in Lemma~\ref{lem:derivative}.
    Then the increment $\Delta = \hat{p}(\bm{p'}) - \hat{p}(\bm{p})$ admits the uniform positive lower bound
    \[
        \Delta \;\ge\; \delta(\varepsilon,\alpha,N) \;>\; 0,
    \]
    where
    \[
        \delta(\varepsilon,\alpha,N)=\min\!\left\{
            \frac{1}{1+(N-1)e^{1/\alpha}},\;
            \frac{\varepsilon}{1+(N-1)e^{\varepsilon/\alpha}}
        \right\}.
    \]
\end{lemma}

\begin{proof}
    Write $a=p_1\in[0,1-\varepsilon]$ and set
    \[
        A=\sum_{i=2}^n p_i e^{p_i/\alpha},\qquad 
        B=\sum_{i=2}^n e^{p_i/\alpha}
    \]
    (with $A=B=0$ when $n=1$). For any $x\in[0,1]$ define the auxiliary function
    \[
        h(x)=\frac{x e^{x/\alpha}+A}{e^{x/\alpha}+B},
    \]
    so that $\hat{p}(\bm{p})=h(a)$ and $\hat{p}(\bm{p}')=h(1)$.

    From Lemma~\ref{lem:derivative}, $h'(x)=\frac{w_1(x)}{\alpha}(\alpha+x-h(x))$ where $w_1(x)=e^{x/\alpha}/(e^{x/\alpha}+B)>0$.
    Every interior critical point satisfies $\alpha+x=h(x)$, and evaluating the second derivative there gives $h''=w_1/\alpha>0$; hence all interior critical points are strict local minima.
    Hence the maximum of $h$ on any subinterval must occur at an endpoint.  

    Consequently, for $a\in[0,1-\varepsilon]$ we have $h(a)\le \max\{h(0),h(1-\varepsilon)\}$, and therefore
    \[
        \Delta = h(1)-h(a) \ge \min\!\bigl\{h(1)-h(0),\; h(1)-h(1-\varepsilon)\bigr\}.
    \]

    We now bound the two differences from below.

    \textbf{Bounding $h(1)-h(0)$.}
    A direct computation gives
    \[
        h(1)-h(0)
        =\frac{e^{1/\alpha}(1+B-A)+A}{(e^{1/\alpha}+B)(1+B)}.
    \]
    Since $p_i\in[0,1]$ implies $A\le B$ and $e^{1/\alpha}>1$, the numerator decreases with $A$.  
    Its minimum on the admissible region is attained at $A=B$, yielding
    \[
        h(1)-h(0)\ge\frac{1}{1+B}.
    \]
    From $p_i\le 1$ we have $B\le (n-1)e^{1/\alpha}\le (N-1)e^{1/\alpha}$, hence
    \[
        h(1)-h(0)\ge\frac{1}{1+(N-1)e^{1/\alpha}}.
    \]

    \textbf{Bounding $h(1)-h(1-\varepsilon)$.}
    Set $x=1-\varepsilon$. Algebraic simplification yields $h(1)-h(x)=$
    \[
        \frac{\varepsilon e^{x/\alpha}(e^{1/\alpha}+B) + (B-A)(e^{1/\alpha}-e^{x/\alpha})}
             {(e^{1/\alpha}+B)(e^{x/\alpha}+B)}.
    \]
    Since $A\le B$ and $e^{1/\alpha}>e^{x/\alpha}$, the second term in the numerator is non‑negative.  
    Hence the numerator is at least $\varepsilon e^{x/\alpha}(e^{1/\alpha}+B)$, and we obtain the lower bound
    \[
        h(1)-h(x)\ge\frac{\varepsilon e^{x/\alpha}}{e^{x/\alpha}+B}.
    \]
    Using $B\le (N-1)e^{1/\alpha}$ and $e^{x/\alpha}=e^{(1-\varepsilon)/\alpha}$,
    \[
        h(1)-h(x)\ge\frac{\varepsilon}{1+(N-1)e^{\varepsilon/\alpha}}.
    \]

    Taking the minimum of these two bounds establishes the claimed inequality.
    The strict positivity follows from $\varepsilon>0$ and the boundedness of all terms.
\end{proof}

\begin{corollary}[Maximum of SWS]\label{cor:maximum}
    Let $\alpha>0$, $n\ge 1$, and fix arbitrary $p_2,\dots,p_n\in[0,1]$.  
    Consider the function
    \[
        \varphi(x) = \hat{p}(x,p_2,\dots,p_n), \qquad x\in[0,1],
    \]
    where $\hat{p}$ is the softmax weighted sum defined in Lemma~\ref{lem:derivative}.
    Then $\varphi$ attains its maximum on $[0,1]$ uniquely at $x=1$.
\end{corollary}

\begin{proof}
    For any $x\in[0,1)$, set $\varepsilon = 1-x > 0$.
    Applying Lemma~\ref{lem:increment} with $N=n$ (the actual size of the bag) to the vector
    $\bm{p} = (x,p_2,\dots,p_n)$ gives
    \[
        \varphi(1) - \varphi(x) = \hat{p}(1,p_2,\dots,p_n) - \hat{p}(x,p_2,\dots,p_n)
        \ge \delta(\varepsilon,\alpha,n) \;>\; 0.
    \]
        
    Hence $\varphi(1) > \varphi(x)$ for every $x<1$, which means $x=1$ is the strict global maximum.
\end{proof}

\subsection{Proof of Theorem~\ref{thm:mil-minimize}}

\begin{lemma}[Loss decrease for correcting false negatives]\label{lem:mil-bias-cont-step1}
Let $\mathcal{X}$ be the instance space, $\mathcal{B}$ the bag space, and $\mathbb{P}$ a distribution over $\mathcal{B}$.
Let $f:\mathcal{X}\to[0,1]$ be an instance‑level classifier.
For a bag $B=(x_1,\dots,x_n)$ with label $y(B)\in\{0,1\}$, define the softmax weighted sum and the expected loss as
\[
   \hat{p}(B)=\mathrm{SWS}_\alpha\bigl(f(x_1),\dots,f(x_n)\bigr), \qquad
   \mathcal{L}(f)=\mathbb{E}_{B\sim\mathbb{P}}\bigl[L_{\mathrm{CE}}\bigl(\hat{p}(B),y(B)\bigr)\bigr].
\]
Let $S^+\subseteq\mathcal{X}^+$ be a set of positive instances satisfying $\mathbb{P}(\mathcal{C}^+(S^+))>0$,
and assume $f(x)<1$ for all $x\in S^+$.
Define a new classifier $f'$ by
\[
f'(x)=
\begin{cases}
1, & x\in S^+,\\
f(x), & \text{otherwise}.
\end{cases}
\]
Then the expected loss strictly decreases, i.e.\ $\mathcal{L}(f')<\mathcal{L}(f)$.
\end{lemma}

\begin{proof}
By Corollary~\ref{cor:maximum}, replacing any $p_i<1$ by $1$ strictly increases $\hat{p}$.  
Applying this repeatedly to all $x\in S^+$ contained in a bag $B\in\mathcal{C}^+(S^+)$ gives $\hat{p}'(B)>\hat{p}(B)$. Hence $-\log\hat{p}'(B)<-\log\hat{p}(B)$ on $\mathcal{C}^+(S^+)$.  
All other bags are unaffected because $f'$ differs from $f$ only on $S^+$, which appears only in bags of $\mathcal{C}^+(S^+)$. Therefore
\[
\mathcal{L}(f)-\mathcal{L}(f') = \mathbb{E}\bigl[ \mathbf{1}_{\mathcal{C}^+(S^+)}\,(\log\hat{p}'-\log\hat{p}) \bigr] > 0,
\]
since the integrand is positive on a set of positive probability. Hence $\mathcal{L}(f')<\mathcal{L}(f)$.
\end{proof}

\begin{lemma}[Loss decrease for correcting false positives]\label{lem:mil-bias-cont-step2}
Let $\mathcal{X},\mathcal{B},\mathbb{P},\hat{p},\mathcal{L}$ be as defined in Lemma~\ref{lem:mil-bias-cont-step1}.  
Suppose that Assumption~\ref{ass:M-restated} (the $M$-assumption) holds.  
Let $N$ be the maximum bag size,  
$u = W_0\!\bigl(\frac{N-1}{e}\bigr)$, and $M$ the constant from that assumption.  
Let $S^-\subseteq\mathcal{X}^-$ be a set of negative instances with
$\mathbb{P}(\mathcal{C}(S^-))>0$, and assume that classifier $f$ satisfies
\[
   f(x)=1 \ \forall x\in\mathcal{X}^+, \qquad
   f(x)\in(0,1] \ \forall x\in S^-, \qquad
   f(x)=0 \ \text{for all other } x\in\mathcal{X}^-.
\]

Define $f'$ by
\[
f'(x)=
\begin{cases}
0, & x\in S^-,\\[2pt]
f(x), & \text{otherwise}.
\end{cases}
\]
Then there exists a temperature $\alpha>0$ such that $\mathcal{L}(f')<\mathcal{L}(f)$.
\end{lemma}

\begin{proof}
Choose $\alpha>0$ small enough so that
\begin{equation}\label{eq:alpha-cond}
     \alpha < \frac{1}{2(u+1)}, \qquad
     -\log\!\bigl(\alpha(2u+1)\bigr) + M\log(1-\alpha u) > 0.
\end{equation}
Such $\alpha$ exists because as $\alpha\to0^+$ the second term tends to $0$ while the first diverges to $+\infty$.
Set $\varepsilon_1 = 1-\alpha u-\alpha$ and partition $S^-$ according to the current predictions:
\[
  S_1 = \{x\in S^- : f(x)>\varepsilon_1\}, \qquad
  S_2 = S^-\setminus S_1 \ \bigl(\text{so } f(x)\le\varepsilon_1\text{ on }S_2\bigr).
\]

Only bags containing instances from $S^-$ change when $f$ is replaced by $f'$. We split those bags into four families:
\[
A_1 = \mathcal{C}^+(S_1), \quad A_2 = \mathcal{C}^-(S_1), \quad
A_3 = \mathcal{C}^+(S_2)\setminus\mathcal{C}^+(S_1), \quad
A_4 = \mathcal{C}^-(S_2)\setminus\mathcal{C}^-(S_1).
\]
Denote by $\hat{p}(B),\hat{p}'(B)$ the softmax weighted sum under $f,f'$. The difference of expected losses decomposes as
\[
\mathcal{L}(f)-\mathcal{L}(f') = \Delta_1+\Delta_2+\Delta_3+\Delta_4,
\]
where
\begin{align*}
\Delta_1 &= \mathbb{E}\bigl[\mathbf{1}_{A_1}(\log\hat{p}'-\log\hat{p})\bigr], \\
\Delta_2 &= \mathbb{E}\bigl[\mathbf{1}_{A_2}\bigl(\log(1-\hat{p}')-\log(1-\hat{p})\bigr)\bigr],\\
\Delta_3 &= \mathbb{E}\bigl[\mathbf{1}_{A_3}(\log\hat{p}'-\log\hat{p})\bigr], \\
\Delta_4 &= \mathbb{E}\bigl[\mathbf{1}_{A_4}\bigl(\log(1-\hat{p}')-\log(1-\hat{p})\bigr)\bigr].
\end{align*}

\noindent\textbf{Bounding $\Delta_1+\Delta_2$.}
If $\mathbb{P}(A_1\cup A_2)=0$, then $\Delta_1=\Delta_2=0$. Otherwise the $M$-assumption gives $\mathbb{P}(A_2)>0$.  
For $B\in A_2$, by Lemma~\ref{lem:minimum-advanced} and $n\le N$,
\[
\hat{p}(B) \ge \varepsilon_1 - \alpha u \ge 1 - 2\alpha u - \alpha,
\]
so $1-\hat{p}(B)\le \alpha(2u+1)$. Since $f'$ assigns $0$ to all instances of $B$, $\hat{p}'(B)=0$, yielding
\[
    \log(1-\hat{p}'(B))-\log(1-\hat{p}(B))
    = -\log(1-\hat{p}(B))
    \ge -\log\!\bigl(\alpha(2u+1)\bigr).
\]

Taking expectations gives 
\[
\Delta_2 \ge \mathbb{P}(A_2)\bigl(-\log(\alpha(2u+1))\bigr).
\]

For $B\in A_1$, Lemma~\ref{lem:minimum-advanced} with $\varepsilon=1$ implies $\hat{p}(B),\hat{p}'(B)\ge 1-\alpha u$, hence
\[
\log\hat{p}'(B)-\log\hat{p}(B) \ge \log(1-\alpha u) 
\]
(note that $\log(1-\alpha u)<0$).
Thus 
\[
\Delta_1 \ge \mathbb{P}(A_1)\log(1-\alpha u).
\]

Combining the two bounds and using the $M$-assumption $\mathbb{P}(A_1)/\mathbb{P}(A_2)\le M$ together with $\log(1-\alpha u)<0$, we obtain
\begin{align*}
\Delta_1+\Delta_2
&\ge \mathbb{P}(A_2)\Bigl[-\log\!\bigl(\alpha(2u+1)\bigr) + \tfrac{\mathbb{P}(A_1)}{\mathbb{P}(A_2)}\log(1-\alpha u)\Bigr] \\
&\ge \mathbb{P}(A_2)\Bigl[-\log\!\bigl(\alpha(2u+1)\bigr) + M\log(1-\alpha u)\Bigr] \\
&>0,
\end{align*}
where the final positivity is exactly the second condition in~\eqref{eq:alpha-cond}.

\smallskip
\noindent\textbf{Positivity of $\Delta_3$.}
For $B\in A_3$, the bag contains a positive instance (predicted as $1$) and at least one instance from $S_2$ with $p_x\le\varepsilon_1$.  
Lemma~\ref{lem:minimum-advanced} again gives $\hat{p}(B)\ge 1-\alpha u$. Using Lemma~\ref{lem:derivative}, for any such $p_x$,
\[
\frac{\partial\hat{p}}{\partial p_x} = \frac{w}{\alpha}\bigl(\alpha+p_x-\hat{p}\bigr)
\le \frac{w}{\alpha}\bigl(\alpha+\varepsilon_1-(1-\alpha u)\bigr)=0.
\]
Hence lowering $p_x$ to $0$ does not decrease $\hat{p}$. Repeating for all $S_2$ instances in $B$ yields $\hat{p}'(B)\ge\hat{p}(B)$, so $\Delta_3\ge0$.

\smallskip
\noindent\textbf{Positivity of $\Delta_4$.}
For $B\in A_4$, $f'$ sets all predictions to $0$, so $\hat{p}'(B)=0$. Under $f$, at least one $S_2$ instance has a positive prediction, therefore $\hat{p}(B)>0$. Consequently,
\[
    \log(1-\hat{p}'(B))-\log(1-\hat{p}(B))
    = -\log(1-\hat{p}(B)) > 0,
\]

which gives $\Delta_4\ge0$, and $\Delta_4>0$ whenever $\mathbb{P}(A_4)>0$.

\smallskip
\noindent\textbf{Combining the four terms.}
If $\mathbb{P}(A_1\cup A_2)>0$, we already have $\Delta_1+\Delta_2>0$ while $\Delta_3,\Delta_4\ge0$, so the sum is strictly positive.  
If $\mathbb{P}(A_1\cup A_2)=0$, then all bags affected by the change belong to $A_3\cup A_4$. The non‑triviality of $S^-$ implies $\mathbb{P}(A_3\cup A_4)=\mathbb{P}(\mathcal{C}(S^-))>0$. Applying the $M$-assumption to $S_2\subseteq\mathcal{X}^-$ yields $\mathbb{P}(A_3)/\mathbb{P}(A_4)\le M$. If $\mathbb{P}(A_4)=0$, this ratio would be infinite, a contradiction. Hence $\mathbb{P}(A_4)>0$ and consequently $\Delta_4>0$, making the total difference positive.

In either case $\mathcal{L}(f)-\mathcal{L}(f')>0$, i.e.\ $\mathcal{L}(f')<\mathcal{L}(f)$.
\end{proof}

\begin{theorem}[Unique global minimizer (Restated)]\label{thm:mil-minimize-restated}
Let $\mathcal{X},\mathcal{B},\mathbb{P},\hat{p},\mathcal{L}$ be as defined in Lemma~\ref{lem:mil-bias-cont-step1}.  
Suppose that Assumption~\ref{ass:M-restated} (the $M$-assumption) holds.  
Let $N$ be the maximum bag size,  
$u = W_0\!\bigl(\frac{N-1}{e}\bigr)$, and $M$ the constant from that assumption.  
Then there exists a sufficiently small temperature $\alpha > 0$ such that, up to a set of instances $S$ whose bag cover set has probability zero, $\mathcal{L}$ is uniquely minimized by the true label function:
\[
f^*(x) = y(x) \in \{0,1\}, \qquad \forall x \notin S.
\]
Equivalently, any global minimizer $f^*$ must satisfy $f^*(x)=y(x)$ for all $x\in\mathcal{X}$, except possibly on a set $S\subseteq\mathcal{X}$ with $\mathbb{P}(\mathcal{C}(S))=0$.
\end{theorem}

\begin{proof}
Choose the temperature $\alpha>0$ such that
\begin{equation}\label{eq:alpha-cond-global}
 \alpha < \frac{1}{2(u+1)}, \qquad
 -\log\!\bigl(\alpha(2u+1)\bigr) + M\log(1-\alpha u) > 0.
\end{equation}
Such $\alpha$ exists because the left‑hand side of the second inequality tends to $+\infty$ as $\alpha\to0^+$, while the first is an explicit positive upper bound.
We fix this $\alpha$ for the remainder of the proof; note that it depends only on $N$ and $M$.

Let $f$ be any global minimizer of $\mathcal{L}$. We show that $f$ must coincide with the true labels everywhere except on a set with zero bag‑cover probability.

\smallskip
\noindent\textbf{Step 1.  Correction of false negatives.}
Set $S^+ := \{ x\in\mathcal{X}^+ \mid f(x) < 1 \}$.
If $\mathbb{P}(\mathcal{C}(S^+)) > 0$, then Lemma~\ref{lem:mil-bias-cont-step1} (which holds for any $\alpha>0$, in particular for our chosen $\alpha$) applied to $S^+$ produces
$f_1(x) = 1$ for $x\in S^+$, $f_1(x)=f(x)$ otherwise, with $\mathcal{L}(f_1) < \mathcal{L}(f)$, contradicting optimality.
Hence $\mathbb{P}(\mathcal{C}(S^+)) = 0$.

Because the bag cover of $S^+$ has probability zero, modifying $f$ on $S^+$ does not change the expected loss.
Define
\[
\tilde{f}(x) := 
\begin{cases}
1, & x\in\mathcal{X}^+,\\
f(x), & x\in\mathcal{X}^-.
\end{cases}
\]
Then $\mathcal{L}(\tilde{f}) = \mathcal{L}(f)$, so $\tilde{f}$ is also a global minimizer.

\smallskip
\noindent\textbf{Step 2.  Correction of false positives.}
Consider the negative instances where $\tilde{f}$ is strictly positive:
$S^- := \{ x\in\mathcal{X}^- \mid \tilde{f}(x) > 0 \}$.
If $\mathbb{P}(\mathcal{C}(S^-)) = 0$, the claim already holds with $S = S^+\cup S^-$.
Assume therefore $\mathbb{P}(\mathcal{C}(S^-)) > 0$.
By construction, $\tilde{f}$ satisfies exactly the hypotheses of Lemma~\ref{lem:mil-bias-cont-step2}:
$\tilde{f}(x)=1$ on $\mathcal{X}^+$, $\tilde{f}(x)\in(0,1]$ on $S^-$, and $\tilde{f}(x)=0$ on $\mathcal{X}^-\setminus S^-$.
Moreover, our chosen $\alpha$ fulfills the two inequalities that are shown (in the proof of Lemma~\ref{lem:mil-bias-cont-step2}) to be sufficient for the strict decrease of the loss when the false positives in $S^-$ are set to $0$.
Consequently, the same construction applied to $\tilde{f}$ and $S^-$ with temperature $\alpha$ yields a classifier
\[
\tilde{f}'(x) = 
\begin{cases}
0, & x\in S^-,\\
\tilde{f}(x), & \text{otherwise},
\end{cases}
\]
satisfying $\mathcal{L}(\tilde{f}') < \mathcal{L}(\tilde{f})$.
This contradicts the global optimality of $\tilde{f}$ (and hence of $f$).
Thus $\mathbb{P}(\mathcal{C}(S^-)) = 0$ as well.

\smallskip
\noindent\textbf{Conclusion.}
Let $S = S^+ \cup S^-$.  We have $\mathbb{P}(\mathcal{C}(S)) = 0$.
For $x\notin S$, if $x\in\mathcal{X}^+$ then $f(x)=1$ (otherwise $x$ would belong to $S^+$), and if $x\in\mathcal{X}^-$ then $f(x)=0$ (otherwise $x\in S^-$).
Thus $f(x)=y(x)$ for all $x\notin S$, which is exactly the statement of the theorem.
\end{proof}

\subsection{Proof of Theorem~\ref{thm:mil-converge}}

\begin{lemma}[Quantitative loss decrease for correcting false negatives]\label{lem:fn-quant}
    Let $\alpha>0$ and $N\in\mathbb{N}$ ($N\ge 1$) be the temperature and maximum bag size.
    Let $\mathcal{X}$, $\mathcal{B}$, $\mathbb{P}$, $\hat{p}(\cdot)$ and $\mathcal{L}$ be as in Lemma~\ref{lem:mil-bias-cont-step1}.
    Let $f$ be a classifier. For a given $\varepsilon\in(0,1]$, let $S^+\subseteq\mathcal{X}^+$ be a set of positive instances such that
    $f(x)\le 1-\varepsilon$ for all $x\in S^+$, and $\mathbb{P}\bigl(\mathcal{C}^+(S^+)\bigr)\ge \delta>0$.
    Define a new classifier $f'$ by
    \[
        f'(x)=
        \begin{cases}
            1, & x\in S^+,\\
            f(x), & \text{otherwise}.
        \end{cases}
    \]
    Then the expected loss decreases by at least a constant times $\delta$:
    \[
        \mathcal{L}(f)-\mathcal{L}(f') \;\ge\; \delta\cdot \Lambda_{\mathrm{FN}}(\varepsilon,\alpha,N) \;>\; 0,
    \]
    where
    \[
        \Lambda_{\mathrm{FN}}(\varepsilon,\alpha,N)=\log\left(1+
        \min\!\left\{
            \frac{1}{1+(N-1)e^{1/\alpha}},\;
            \frac{\varepsilon}{1+(N-1)e^{\varepsilon/\alpha}}
        \right\}\right).
    \]
\end{lemma}

\begin{proof}
    Denote by $\delta_0=\delta_0(\varepsilon,\alpha,N)$ the increment lower bound from Lemma~\ref{lem:increment}, i.e.
    \[
        \delta_0 = \min\!\left\{
            \frac{1}{1+(N-1)e^{1/\alpha}},\;
            \frac{\varepsilon}{1+(N-1)e^{\varepsilon/\alpha}}
        \right\}>0.
    \]
    Let $B=(x_1,\dots,x_n)\in\mathcal{C}^+(S^+)$. 
    By definition $B$ is a positive bag ($y(B)=1$) and contains at least one instance $x_i\in S^+$ with $f(x_i)\le 1-\varepsilon$.
    Write $\bm{p}=(f(x_1),\dots,f(x_n))$ and let $\bm{p}'$ be the vector obtained after replacing every such $f(x_i)$ by $1$.
    Repeated application of Lemma~\ref{lem:increment}
    shows that
    \[
        \hat{p}'(B) - \hat{p}(B) \;\ge\; \delta_0,
    \]
    because each single replacement increases the softmax weighted sum by at least $\delta_0$, and the sum is non‑decreasing under further replacements.
    
    Since the bag is positive, its contribution to the loss is $-\log\hat{p}(B)$ under $f$ and $-\log\hat{p}'(B)$ under $f'$.
    The decrease in loss for this bag is
    \[
        \log\frac{\hat{p}'(B)}{\hat{p}(B)}
        = \log\!\left(1+\frac{\hat{p}'(B)-\hat{p}(B)}{\hat{p}(B)}\right)
        \ge \log\!\left(1+\delta_0\right),
    \]
    where we used $\hat{p}(B)\le 1$ and the monotonicity of the logarithm.
    
    Bags outside $\mathcal{C}^+(S^+)$ are left unchanged by $f'$, hence their loss contribution is identical.
    Taking expectations,
    \[
        \mathcal{L}(f)-\mathcal{L}(f')
        = \mathbb{E}\Bigl[\mathbf{1}_{\mathcal{C}^+(S^+)}(B)\,
            \bigl(\log\frac{\hat{p}'(B)}{\hat{p}(B)}\bigr)\Bigr]
        \ge \mathbb{P}\bigl(\mathcal{C}^+(S^+)\bigr)\cdot \log(1+\delta_0).
    \]

    By hypothesis $\mathbb{P}(\mathcal{C}^+(S^+))\ge\delta$, so
    \[
        \mathcal{L}(f)-\mathcal{L}(f') \;\ge\; \delta\cdot\log(1+\delta_0)= \delta\cdot\Lambda_{\mathrm{FN}}(\varepsilon,\alpha,N).
    \]
    The strict positivity follows from $\delta_0>0$ and $\delta>0$.
\end{proof}

\begin{lemma}[Quantitative loss decrease for correcting false positives]\label{lem:fp-quant}
    Let $\mathcal{X},\mathcal{B},\mathbb{P},\hat{p},\mathcal{L}$ be as in Lemma~\ref{lem:mil-bias-cont-step1}.  
    Suppose Assumption~\ref{ass:M-restated} (the $M$-assumption) holds.  
    Let $N$ be the maximum bag size, $u = W_0\bigl(\frac{N-1}{e}\bigr)$, and $M$ the constant from that assumption.  
    Let $S^-\subseteq\mathcal{X}^-$ be a set of negative instances with
    $\mathbb{P}(\mathcal{C}(S^-))>0$, and assume that the classifier $f$ satisfies
    \[
        f(x)=1 \ \forall x\in\mathcal{X}^+, \qquad
        f(x)\in(0,1] \ \forall x\in S^-, \qquad
        f(x)=0 \ \text{for all other } x\in\mathcal{X}^-.
    \]
    For a given $\varepsilon\in(0,1]$, define the subset
    $\tilde{S} := \{x\in S^- \mid f(x)\ge \varepsilon\}$,
    and suppose that its bag cover set has probability
    $\delta := \mathbb{P}\bigl(\mathcal{C}(\tilde{S})\bigr) > 0$.
    Define the corrected classifier
    \[
        f'(x)=
        \begin{cases}
            0, & x\in S^-,\\
            f(x), & \text{otherwise}.
        \end{cases}
    \]
    Then there exists a temperature $\alpha>0$ depending only on $N$ and $M$ such that
    \[
        \mathcal{L}(f)-\mathcal{L}(f') \;\ge\; \delta\cdot\Lambda_{\mathrm{FP}}(\varepsilon,\alpha,N,M) \;>\; 0,
    \]
    where
    \[
    \Lambda_{\mathrm{FP}}(\varepsilon,\alpha,N,M)
    =\frac{1}{1+M}\min\Bigl\{
      -\log\!\bigl(\alpha(2u+1)\bigr)+M\log(1-\alpha u),\;
      -\log\!\bigl(1-\hat{p}_{\min}(\varepsilon,\alpha,N)\bigr)
    \Bigr\},
    \]
    and $\hat{p}_{\min}(\varepsilon,\alpha,N)$ is the minimum softmax weighted sum under the constraint $\max_i p_i\ge\varepsilon$,
    given by Lemma~\ref{lem:minimum-advanced}.
\end{lemma}

\begin{proof}
    Choose $\alpha>0$ small enough so that
    \[
        \alpha<\frac{1}{2(u+1)}, \qquad
        \Lambda_0:=-\log\!\bigl(\alpha(2u+1)\bigr)+M\log(1-\alpha u)>0.
    \]
    Such an $\alpha$ exists because as $\alpha\to0^+$ the second term tends to $0$ while the first diverges to $+\infty$.
    Set $\varepsilon_1 = 1-\alpha u-\alpha$ and partition $S^-$ into
    $S_1 = \{x\in S^-:f(x)>\varepsilon_1\}$ and $S_2 = S^-\setminus S_1$.

    Following the decomposition in Lemma~\ref{lem:mil-bias-cont-step2}, the change in expected risk is written as
    \[
        \mathcal{L}(f)-\mathcal{L}(f') = \Delta_1+\Delta_2+\Delta_3+\Delta_4,
    \]
    where the four terms correspond to the bag families
    \[
        A_1 = \mathcal{C}^+(S_1), \quad A_2 = \mathcal{C}^-(S_1), \quad
        A_3 = \mathcal{C}^+(S_2)\setminus\mathcal{C}^+(S_1), \quad
        A_4 = \mathcal{C}^-(S_2)\setminus\mathcal{C}^-(S_1),
    \]
    defined with respect to $S_1,S_2$.

    \noindent\textbf{Bounds for $\Delta_1,\Delta_2,\Delta_3$.}
    Exactly as in the proof of Lemma~\ref{lem:mil-bias-cont-step2}, the conditions on $f$ guarantee
    \begin{align}
        \Delta_1+\Delta_2 &\ge \mathbb{P}(A_2)\,\Lambda_0, \label{eq:fp-quant-d12a2}\\
        \Delta_3 &\ge 0. \label{eq:fp-quant-d3a2}
    \end{align}
    Since $\mathbb{P}(A_2) \ge \mathbb{P}(A_2\cap\mathcal{C}^-(\tilde{S}))$ and $\Lambda_0 > 0$, we further have
    \begin{equation}\label{eq:fp-quant-d12b2}
        \Delta_1+\Delta_2 \;\ge\; \mathbb{P}\bigl(A_2\cap\mathcal{C}^-(\tilde{S})\bigr)\,\Lambda_0.
    \end{equation}

    \noindent\textbf{Bound for $\Delta_4$.}
    As analyzed in Lemma~\ref{lem:mil-bias-cont-step2}, for any $B\in A_4$, 
    \[
        \log(1-\hat{p}'(B))-\log(1-\hat{p}(B))
        = -\log(1-\hat{p}(B)) > 0.
    \]
    Hence $\Delta_4 = \mathbb{E}[\mathbf{1}_{A_4}(B)\,(-\log(1-\hat{p}(B)))]\ge 0$.

    Now consider any $B\in A_4\cap\mathcal{C}^-(\tilde{S})$. Such a bag additionally contains at least one instance from $\tilde{S}$, so $\max_{x\in B} f(x)\ge\varepsilon$.
    Applying Lemma~\ref{lem:minimum-advanced} to the predictions on $B$ gives
    \[
        \hat{p}(B) \;\ge\; \hat{p}_{\min}(\varepsilon,\alpha,n) \;\ge\; \hat{p}_{\min}(\varepsilon,\alpha,N),
    \]
    where $n\le N$ is the size of $B$ and $\hat{p}_{\min}$ is non‑increasing in $n$.
    Consequently,
    \[
        -\log(1-\hat{p}(B)) \;\ge\; -\log\!\bigl(1-\hat{p}_{\min}(\varepsilon,\alpha,N)\bigr) > 0.
    \]
    Restricting the expectation defining $\Delta_4$ to the subset $A_4\cap\mathcal{C}^-(\tilde{S})$ yields
    \begin{equation}\label{eq:fp-quant-d4}
        \Delta_4 \;\ge\; \mathbb{P}(A_4 \cap \mathcal{C}^-(\tilde{S}))\,\Lambda_1, \qquad
        \Lambda_1:=-\log\!\bigl(1-\hat{p}_{\min}(\varepsilon,\alpha,N)\bigr).
    \end{equation}

    \noindent\textbf{Combining the estimates.}
    From \eqref{eq:fp-quant-d12b2}, \eqref{eq:fp-quant-d3a2} and \eqref{eq:fp-quant-d4} we obtain
    \begin{align*}
        \mathcal{L}(f)-\mathcal{L}(f')
        &\ge \min\{\Lambda_0,\Lambda_1\}\,
           \bigl(\mathbb{P}(A_2\cap\mathcal{C}^-(\tilde{S})) + \mathbb{P}(A_4\cap\mathcal{C}^-(\tilde{S}))\bigr)\\
        &= \min\{\Lambda_0,\Lambda_1\}\,
           \mathbb{P}\bigl(\mathcal{C}^-(\tilde{S})\bigr).
    \end{align*}

    Finally, the $M$-assumption applied to $\tilde{S}\subseteq\mathcal{X}^-$ gives
    \[
        \mathbb{P}\bigl(\mathcal{C}(\tilde{S})\bigr) \;\le\; (1+M)\,\mathbb{P}\bigl(\mathcal{C}^-(\tilde{S})\bigr),
    \]
    i.e.\ $\mathbb{P}(\mathcal{C}^-(\tilde{S})) \ge \delta/(1+M)$.  Substituting back,
    \[
        \mathcal{L}(f)-\mathcal{L}(f') \ge
        \frac{\delta}{1+M}\,\min\{\Lambda_0,\Lambda_1\}
        =\delta\cdot\Lambda_{\mathrm{FP}}(\varepsilon,\alpha,N,M) \;>\;0,
    \]
    which completes the proof.
\end{proof}

\begin{theorem}[Convergence in probability (Restated)]\label{thm:mil-convergence-restated}
    Let Assumption~\ref{ass:M} hold. 
    Let $N$ be the maximum allowed bag size,
    $u = W_0\bigl(\frac{N-1}{e}\bigr)$, and $M$ the constant from the $M$-assumption.
    Let $f^*(x)=y(x)$ be the Bayes optimal classifier, and let $\{f_l\}$ be a sequence of instance classifiers.
    Define the expected loss $\mathcal{L}(f)$ as in Lemma~\ref{lem:mil-bias-cont-step1}.
    Then there exists a sufficiently small temperature $\alpha > 0$ such that,
    if $\lim_{l\to\infty} \mathcal{L}(f_l) = \mathcal{L}(f^*)$,
    then $\{f_l\}$ converges to $f^*$ in probability.
    Equivalently, for any $\varepsilon > 0$,
    \[
        \lim_{l\to\infty} \mathbb{P}\Bigl(\mathcal{C}\bigl(\{x:|f_l(x)-f^*(x)|\ge \varepsilon\}\bigr)\Bigr) = 0.
    \]
\end{theorem}

\begin{proof}
    Choose a temperature $\alpha>0$ small enough to satisfy both
    \begin{equation}\label{eq:alpha-main}
        \alpha < \frac{1}{2(u+1)}, \qquad
        \Lambda_0 := -\log\!\bigl(\alpha(2u+1)\bigr) + M\log(1-\alpha u) > 0.
    \end{equation}
    Such $\alpha$ exists because $\Lambda_0\to+\infty$ as $\alpha\to0^+$ while the first bound is an explicit positive constant.
    In the sequel we work with this fixed $\alpha$.
    Define the two positive constants
    \[
        \Lambda_1(\varepsilon) := \Lambda_{\mathrm{FN}}(\varepsilon,\alpha,N), \qquad
        \Lambda_2(\varepsilon) := \Lambda_{\mathrm{FP}}(\varepsilon,\alpha,N,M),
    \]
    where $\Lambda_{\mathrm{FN}}$ and $\Lambda_{\mathrm{FP}}$ are given by Lemma~\ref{lem:fn-quant} and Lemma~\ref{lem:fp-quant} respectively.
    Both $\Lambda_1(\varepsilon)$ and $\Lambda_2(\varepsilon)$ are strictly positive for every $\varepsilon\in(0,1]$.

    Fix arbitrary $\varepsilon,\delta > 0$ and set
    \[
        \xi := \delta \cdot \min\{\Lambda_1(\varepsilon),\,\Lambda_2(\varepsilon)\} > 0.
    \]
    We show that for any classifier $f$,
    \begin{equation}\label{eq:key-ineq}
        \mathbb{P}\Bigl(\mathcal{C}\bigl(\{x:|f(x)-f^*(x)|\ge \varepsilon\}\bigr)\Bigr) \ge \delta
        \quad\Longrightarrow\quad
        \mathcal{L}(f) - \mathcal{L}(f^*) \ge \xi.
    \end{equation}
    Once this implication is proved, the convergence follows immediately:
    if $\mathcal{L}(f_l)\to\mathcal{L}(f^*)$, then for sufficiently large $l$ we have
    $\mathcal{L}(f_l)-\mathcal{L}(f^*) < \xi$, which forces the left‑hand side of~\eqref{eq:key-ineq} to be smaller than $\delta$.
    Hence the bag‑cover probability of the $\varepsilon$-deviating instances tends to $0$.

    \smallskip
    \noindent\textit{Proof of~\eqref{eq:key-ineq}.}
    Let $f$ be a classifier with $\mathbb{P}(\mathcal{C}(S))\ge\delta$, where $S:=\{x:|f(x)-f^*(x)|\ge\varepsilon\}$.
    Decompose $S$ into false negatives and false positives:
    \[
        S_{\mathrm{FN}} := \{x\in\mathcal{X}^+ : f(x)\le 1-\varepsilon\}, \qquad
        S_{\mathrm{FP}} := \{x\in\mathcal{X}^- : f(x)\ge \varepsilon\}.
    \]
    Clearly $S = S_{\mathrm{FN}}\cup S_{\mathrm{FP}}$ and the union is disjoint.
    Denote $\delta_1 = \mathbb{P}(\mathcal{C}(S_{\mathrm{FN}}))$ and $\delta_2 = \mathbb{P}(\mathcal{C}(S_{\mathrm{FP}}))$.
    By subadditivity of $\mathbb{P}(\mathcal{C}(\cdot))$,
    \[
        \delta \le \mathbb{P}(\mathcal{C}(S)) \le \delta_1+\delta_2.
    \]

    We now modify $f$ to the Bayes classifier $f^*$ in three steps and track the loss decrease.

    \textbf{Step 1: Correct significant false negatives.}
    Define $f_1$ by setting $f_1(x)=1$ for $x\in S_{\mathrm{FN}}$ and $f_1(x)=f(x)$ otherwise.
    If $\delta_1>0$, Lemma~\ref{lem:fn-quant} (applicable because $f(x)\le 1-\varepsilon$ on $S_{\mathrm{FN}}$) yields
    \[
        \mathcal{L}(f) - \mathcal{L}(f_1) \;\ge\; \delta_1\,\Lambda_1(\varepsilon).
    \]
    If $\delta_1=0$ the inequality trivially holds as $0\ge0$.

    \textbf{Step 2: Correct remaining false negatives.}
    Define $f_2$ by $f_2(x)=1$ for all $x\in\mathcal{X}^+$ and $f_2(x)=f_1(x)$ for $x\in\mathcal{X}^-$.
    Since $f_2$ only raises predictions of positive instances from possibly values $>1-\varepsilon$ to $1$,
    Lemma~\ref{lem:mil-bias-cont-step1} implies $\mathcal{L}(f_2)\le\mathcal{L}(f_1)$.
    Hence
    \[
        \mathcal{L}(f) - \mathcal{L}(f_2) \;\ge\; \delta_1\,\Lambda_1(\varepsilon).
    \]

    \textbf{Step 3: Correct false positives.}
    Now $f_2$ satisfies precisely the prerequisites of Lemma~\ref{lem:fp-quant}:
    $f_2(x)=1$ on $\mathcal{X}^+$, $f_2(x)\in(0,1]$ on $S^-:=\{x\in\mathcal{X}^-:f_2(x)>0\}$,
    and $f_2(x)=0$ on the rest of $\mathcal{X}^-$.
    The set $S_{\mathrm{FP}}$ is exactly $\tilde{S}$ in that lemma,
    and its bag‑cover probability is $\delta_2$.
    Applying Lemma~\ref{lem:fp-quant} with the already chosen $\alpha$ (which satisfies~\eqref{eq:alpha-main}) gives
    \[
        \mathcal{L}(f_2) - \mathcal{L}(f') \;\ge\; \delta_2\,\Lambda_2(\varepsilon),
    \]
    where $f'$ is the classifier obtained by setting all instances in $S^-$ to $0$.
    Note that $f'$ assigns $1$ to all positive instances and $0$ to all negative ones,
    i.e.\ $f'=f^*$.  Hence $\mathcal{L}(f')=\mathcal{L}(f^*)$.

    \textbf{Combining the three steps.}
    Summing the inequalities,
    \[
        \mathcal{L}(f) - \mathcal{L}(f^*) \ge
        \delta_1\,\Lambda_1(\varepsilon) + \delta_2\,\Lambda_2(\varepsilon)
        \ge (\delta_1+\delta_2)\,\min\{\Lambda_1(\varepsilon),\Lambda_2(\varepsilon)\}
        \ge \delta\,\min\{\Lambda_1(\varepsilon),\Lambda_2(\varepsilon)\}=\xi.
    \]
    This establishes~\eqref{eq:key-ineq} and completes the proof.
\end{proof}

\end{document}